\documentclass{article} 
\usepackage{iclr2025_conference,times}

\usepackage{amsmath,amsfonts,bm}

\def\eqref#1{equation~\ref{#1}}

\def\1{\bm{1}}

\DeclareMathAlphabet{\mathsfit}{\encodingdefault}{\sfdefault}{m}{sl}
\SetMathAlphabet{\mathsfit}{bold}{\encodingdefault}{\sfdefault}{bx}{n}

\usepackage{hyperref}
\usepackage{url}

\usepackage[utf8]{inputenc} 
\usepackage[T1]{fontenc}    
\usepackage{booktabs}       
\usepackage{amsfonts}       
\usepackage{nicefrac}       
\usepackage{microtype}      
\usepackage{xcolor}         
\usepackage{mathtools}
\usepackage{shuffle}  
\usepackage{booktabs}       
\usepackage{enumitem}
\usepackage{amsfonts}       
\usepackage{nicefrac}       
\usepackage{microtype}
\usepackage{amsthm}
\usepackage{amsmath}
\usepackage{subfigure}
\usepackage{adjustbox}
\usepackage{graphicx}
\usepackage{multirow}
\usepackage{amssymb}
\usepackage[capitalize,noabbrev]{cleveref}

\newtheorem{definition}{Definition}[section]
\newtheorem{theorem}{Theorem}[section]
\newtheorem{proposition}{Proposition}[theorem]
\newtheorem{lemma}{Lemma}[theorem]
\newtheorem*{remark}{Remark}
\newtheorem{example}{Example}[section]

\title{SigDiffusions: Score-Based Diffusion Models for Time Series via Log-Signature Embeddings}

\author{Barbora Barancikova\\
Department of Computing\\
Imperial College London\\
\And
Zhuoyue Huang\\
Department of Mathematics\\
Imperial College London\\
\And
Cristopher Salvi\\
Department of Mathematics\\
Imperial College London\\
}

\iclrfinalcopy 
\begin{document}

\maketitle

\begin{abstract}
Score-based diffusion models have recently emerged as state-of-the-art generative models for a variety of data modalities. Nonetheless, it remains unclear how to adapt these models to generate long multivariate time series. Viewing a time series as the discretisation of an underlying continuous process, we introduce  \texttt{SigDiffusion}, a novel diffusion model operating on log-signature embeddings of the data. The forward and backward processes gradually perturb and denoise log-signatures while preserving their algebraic structure. To recover a signal from its log-signature, we provide new closed-form inversion formulae expressing the coefficients obtained by expanding the signal in a given basis (e.g. Fourier or orthogonal polynomials) as explicit polynomial functions of the log-signature. Finally, we show that combining \texttt{SigDiffusions} with these inversion formulae results in high-quality long time series generation, competitive with the current state-of-the-art on various datasets of synthetic and real-world examples.
\end{abstract}

\section{Introduction}

Time series generation has gained significant attention in recent years due to the growing demand for high-quality data augmentation in fields such as healthcare \citep{trottet2023generative} and finance \citep{hwang2023augmentation}. Since sampling rates are often arbitrary and non-uniform, it is natural to assume that the data is collected from measurements of some underlying physical system evolving in continuous time. This perspective calls for modelling tools capable of processing temporal signals as continuous functions of time. We will often refer to such functions as \emph{paths}.

The idea of representing a path via its iterated integrals has been the object of numerous mathematical studies, including geometry \citep{chen1957integration, chen1958integration} control theory \citep{fliess1983algebraic}, and stochastic analysis \citep{lyons1998differential}. The collection of these iterated integrals is known as the \emph{signature} of a path. Thanks to its rich algebraic and analytic properties summarised in \cref{sec:sig-diff}, the signature is an efficient-to-compute, universal feature map for time series evolving in continuous time. As a result, signature
methods have recently become mainstream in many areas of machine learning dealing with irregular
time series, from deep learning \citep{kidger2019deep, morrill2021neural, cirone2023neural, cirone2024theoretical} to kernel methods \citep{salvi2021signature, lemercier2021distribution, issa2024non}, with applications in quantitative finance \citep{arribas2020sig, salvi2021higher, horvath2023optimal, pannier2024path}, cybersecurity \citep{cochrane2021sk}, weather forecasting \citep{lemercier2021siggpde}, and causal inference \citep{manten2024signature}. For a concise summary of this topic, we refer the interested reader to \citet{fermanian2023new}.

Score-based diffusion models have emerged as a powerful framework for modelling complex distributions in computer vision, audio, and text \citep{song2020score, bilovs2023modeling, popov2021grad, cai2020learning, voleti2022mcvd}. These models perturb the observed data distribution through a forward diffusion process, adding noise until an easy-to-sample base distribution is reached. A new sample from the learned data distribution is then generated by reversing the noising process using a learned \textit{score} function of the data conditional on the noise level. Despite recent efforts summarised in \cref{sec:related-work}, it remains unclear how to adapt score-based diffusion models for generating long signals in continuous time.

\paragraph{Contributions} In this paper, we make use of the \textit{log-signature}, a compressed version of the signature, as a parameter-free Lie algebra embedding for time series. In \cref{sec:diff-sig}, we introduce \texttt{SigDiffusion}, a new diffusion model that gradually perturbs and denoises log-signatures preserving their algebraic structure. To recover a path from its log-signature embedding, we provide novel
closed-form inversion formulae in \cref{inversion}. Notably, we prove that the coefficients of the expansion
of a path in a given basis, such as Fourier or orthogonal polynomials, can be expressed as explicit
polynomial functions on the log-signature. Our results provide a major improvement over existing signature inversion algorithms (see \cref{subsection:inv_eval}) which suffer from scalability issues and are only effective on simple examples of short piecewise-linear paths. In \cref{experiments}, we demonstrate how the combination of \texttt{SigDiffusions} with our inversion formulae provides a time series generative approach, competitive with state-of-the-art diffusion models for temporal data on various datasets of synthetic and real-world examples.

\section{Generating Log-Signatures with Score-Based Diffusion Models}\label{sec:sig-diff}
We begin this section by recalling the relevant background material before introducing our \texttt{SigDiffusion} model. We will limit ourselves to reporting only the key properties of signatures and the notation necessary for the inversion formulae in \cref{inversion}. Additional examples of signature computations can be found in Appendix \ref{sec:ex-comp}.
\subsection{The (log)signature}
\label{section:log-sig}
Let $x : [0,1] \to \mathbb{R}^d$ be a smooth $d$-dimensional time series defined on a time interval $[0,1]$. We will equivalently refer to this object as a \textit{path}.  The \emph{step-$n$ signature} $S^{\leq n}(x)$ of $x$ is defined as the following collection of iterated integrals
\begin{equation}\label{eqn:step-n-sig-simple}
S^{\leq n}(x) = \left(1, S_1(x)  ,\ldots, S_n(x)\right)
\end{equation}
where
\begin{equation*}
    S_k(x)=\int_{0\leq t_{1}<...<t_{k}\leq 1} dx_{t_{1}} \otimes ... \otimes dx_{t_{k}} \quad \text{for } 1 \leq k \leq n
\end{equation*}
and $\otimes$ denotes the tensor product. Intuitively, one can view the signature as a set of tensors of increasing dimension, where the value of the $m$-th tensor at the index $i_1, i_2, \ldots, i_m$ represents a measure of ``interaction'' between the $i_1, i_2, \ldots, i_m$-th channels of $x$. This makes the signature transform particularly effective at capturing information about the shape of multivariate time series.
\begin{example}
    Assume $d=2$, and denote the two channels of $x$ as $x=(x^1, x^2)$. Then $S_1(x),  S_2(x)$ are tensors with shape $[2]$ and $[2,2]$ respectively
    \begin{equation*}
        S_1(x) = \int_0^1 dx_{t_{1}} = \begin{pmatrix}
                                        \int_0^1 dx_{t_{1}}^1 \\
                                        \int_0^1 dx_{t_{1}}^2
                                        \end{pmatrix},
    \end{equation*}
    \begin{equation*}
                S_2(x) = \int_0^1 \int_0^{t_2} dx_{t_{1}} \otimes dx_{t_{2}} = \begin{pmatrix}
                                                                        \int_0^1 \int_0^{t_2} dx_{t_{1}}^1 dx_{t_{2}}^1 & \int_0^1 \int_0^{t_2}  dx_{t_{1}}^2 dx_{t_{2}}^1 \\
                                                                        \int_0^1 \int_0^{t_2}  dx_{t_{1}}^1 dx_{t_{2}}^2 & \int_0^1 \int_0^{t_2}  dx_{t_{1}}^2 dx_{t_{2}}^2
                                                                        \end{pmatrix}.
    \end{equation*}
    
\end{example}
Denoting the standard basis of $\mathbb R^d$ as $e_1,e_2,...,e_d$ we define a basis of the space of $k$-dimensional tensors as
$$e_{i_1i_2...i_k} = e_{i_1} \otimes e_{i_2} \otimes...\otimes e_{i_k}, \quad \text{for} \quad 1 \leq i_1,...,i_k \leq d \text{ and } 0 \leq k \leq n.$$ We refer to these basis elements as \textit{words}.
In \cref{inversion}, we will make use of the notation $\langle e_{i_1i_2...i_k}, S^{\leq n}(x) \rangle \in \mathbb R$ to extract the $(i_1,...,i_k)^{th}$ element of the $k$-th signature tensor $S_k(x)$. 

Words can be manipulated by two key operations: the \emph{shuffle product} $\shuffle$ and \emph{right half-shuffle product} $\succ$. The shuffle product of two words of length $r$ and $s$ (with $r + s \leq n$) is defined as the sum over the $\binom{r+s}{s}$ ways of interleaving the two words. For a formal definition, we refer readers to \citet[Section 1.4]{reutenauer2003free}. Much of the internal structure of the signature is characterized by the \textit{shuffle identity} (see \cref{thm:shuffle}), which uses the \textit{shuffle} and \textit{half-shuffle products} to describe the relationship between elements of higher and lower-order signature tensors. This identity is crucial in our derivation of the inversion formulae in \cref{sec:proofs}. A rigorous algebraic explanation of these concepts is provided in \cref{subsection:algebraic_setup}.

Moreover, it turns out that the space of signatures has the structure of a \emph{step-$n$ free nilpotent Lie group} $\mathcal{G}^n(\mathbb R^d)$. 
We denote by $\mathcal{L}^n(\mathbb R^d)$ the unique Lie algebra associated with $\mathcal{G}^n(\mathbb R^d)$, and we call its elements \textit{log-signatures}. $\mathcal{G}^n(\mathbb R^d)$ is the image of the $\mathcal{L}^n(\mathbb R^d)$ under
the exponential map
\begin{equation}\label{eqn:exp-L}
    \mathcal{G}^n(\mathbb R^d) = \exp(\mathcal{L}^n(\mathbb R^d))
\end{equation}
where, in the case of signatures, $\exp$ denotes the tensor exponential defined in \cref{appendix:log-sig}. Furthermore, one can use the tensor logarithm (see \cref{eqn:log}) to convert log-signatures to signatures. These two operations are mutually inverse.

 We note that the Lie algebra $\mathcal{L}^n(\mathbb R^d)$ is a vector space of dimension $\beta(d, n)$ with
\begin{equation*}
    \beta(d, n) = \sum_{k = 1}^n \frac{1}{k} \sum_{i | k} \mu\left(\frac{k}{i}\right) d^i,
\end{equation*}
where $\mu$ is the M{\"o}bius function \citep{reutenauer2003free}. Crucially, the Lie algebra is isomorphic to the Euclidean space $\mathbb R^{\beta(d, n)}$, which motivates the diffusion model architecture in \cref{sec:diff-sig}. 

\subsection{Signature as a time series embedding}
\label{section:sig-properties}

The (log)signature exhibits additional properties making it an especially interesting object in the context of generative modelling for sequential data. In this section we summarise such properties without providing technical details, as these have been discussed at length in various texts in the literature. For a thorough review, we refer the interested reader to \citet[Chapter 1]{cass2024lecture}.

At first glance, computing the (log)signature seems intractable for general time series. However, an elegant and \textbf{efficient computation} is possible due to \emph{Chen's relation}

\begin{lemma}[Chen's relation]\label{thm:chen}
    For any two smooth paths $x,y : [0,1] \to \mathbb R^d$ the following holds
    \begin{equation}\label{eqn:chen}
        S^{\leq n}(x*y) = S^{\leq n}(x) \cdot S^{\leq n}(y),
    \end{equation}
    where $*$ denotes path-concatenation, and $\cdot$ is the signature tensor product defined in \cref{eqn:tensor-product}.
\end{lemma}
Combining Chen's relation with the fact that the signature of a linear path is simply the tensor exponential of its increment (see \cref{example_linear_sig}) provides us with an efficient algorithm for computing signatures of piecewise linear paths. This approach eliminates the need to calculate integrals when computing signature embeddings. See Appendix \ref{sec:ex-comp} for simple examples of computations.

Furthermore, the (log)signature  is \emph{invariant under reparameterisations}. This property essentially allows the signature transform to act as a filter that removes an infinite dimensional group of symmetries given by time reparameterisations. The action of reparameterising a path can be thought of as sampling its observations at a different frequency, resulting in \textbf{robustness to irregular sampling}. 
Another important property of the signature is the \textbf{factorial decay in the magnitude of its coefficients} \citep[Proposition 1.2.3]{cass2024lecture}. This fast decay implies that truncating the signature at a sufficiently high level retains the bulk of the critical information about the underlying path.

The signature is \emph{unique} for certain classes of paths, ensuring \textbf{a one-to-one identifiability with the underlying path}. An example of such classes is given by paths which share an identical, strictly monotone coordinate and are started at the same origin. More general examples are discussed in \citet[Section 4.1]{cass2024lecture}. This property is important if one is interested, as we are, in recovering the path from its signature. Yet, providing a viable algorithm for inverting the signature has, until now, been challenging; valid but non-scalable solutions have been proposed only for special classes of piecewise linear paths \citep{chang2019insertion, fermanian2023insertion, kidger2019deep}. In \cref{inversion} we provide new closed-form inversion formulae that address this limitation.

\subsection{Diffusion models on log-signature embeddings}\label{sec:diff-sig}

As described in \cref{section:log-sig}, any element of $\mathcal{G}^n(\mathbb R^d)$ corresponds to the step-$n$ signature of a smooth path. Taking the tensor logarithm in Equation (\ref{eqn:log}) then implies that an arbitrary element of $\mathcal{L}^n(\mathbb R^d)$ corresponds to the \emph{step-$n$ log-signature} of a smooth path. Because the Lie algebra $\mathcal{L}^n(\mathbb R^d)$ is a linear space, adding two log-signatures yields another log-signature. Furthermore, the dimensionality $\beta(d, n)$ of $\mathcal{L}^n(\mathbb R^d)$ is strictly smaller than $\frac{d^{n+1}-1}{d-1}$, making the log-signature a more compact representation of a path than the signature while retaining the same information. We can leverage these two properties to run score-based diffusion models on $\mathcal{L}^n(\mathbb R^d)$, followed by an explicit signature inversion, which we discuss in the next section. \cref{fig:sigdiffusions_diagram} presents an overview of the proposed idea, which we demonstrate experimentally in \cref{experiments} to be an efficient and high-quality method for generating multivariate time series. 

We briefly recall that score-based diffusion models progressively corrupt data with noise until it reaches a tractable form, then learn to reverse this process to generate new samples from the underlying data distribution $p(\textbf{x})$. A neural network is used to estimate the gradient of the log probability density, known as the \textit{score} \citep{song2019generative}, $s_\theta(t, \textbf{x})\approx \nabla_{\textbf{x}}\log p_t(\textbf{x})$ at each injected noise level $t$. We model the forward data perturbation process with a stochastic differential equation (SDE) of the form
\begin{equation}\label{eq:ODE}
    d\textbf{x} = -\frac{1}{2}\beta(t)\textbf{x}dt + \sqrt{\beta(t)}d\textbf{w}, \text{ }\textbf{x}(0)\sim \text{data} 
\end{equation}
where $\beta(t)$ is linear over $t\in[0,1]$. The corresponding reverse diffusion process follows a reverse-time SDE \citep{anderson1982reverse} $d\textbf{x} = \left[-\frac{1}{2} \beta(t) \textbf{x} - \beta(t) \nabla_{\textbf{x}} \log p_t(\textbf{x}) \right] dt + \sqrt{\beta(t)} d\overline{\textbf{w}}$, where $t$ flows backwards from $T$ to $0$ and $\overline{\textbf{w}}$ is Brownian motion with a negative time step $dt$. The initial point $\textbf{x}(T)$ is sampled from a standard Gaussian distribution.  Following previous work \citep{ho2020denoising, bilovs2023modeling, yuan2024diffusion}, we use a simple transformer architecture with sinusoidal positional embeddings of $t$ to model $s_\theta$.

\begin{figure}
    \centering
    \includegraphics[width=1\linewidth]{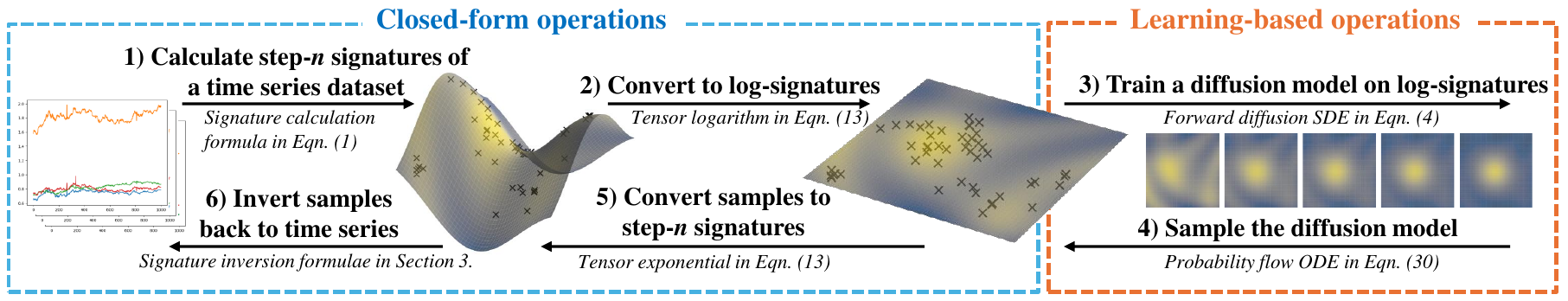}
    \caption{\texttt{SigDiffusions} pipeline. The signatures of a time series dataset are points distributed in a non-Euclidean space (Lie group). Converting to log-signatures maps them to a Euclidean space (Lie algebra) where standard diffusion models operate. Calculating the log-signature embedding and its inverse (blue box) are fully deterministic operations, which greatly simplifies the learning task. The log-signatures serve as inputs to a score-based diffusion model (orange box). Step 6 is enabled by our newly derived closed-form inversion formulae.}
    \label{fig:sigdiffusions_diagram}
\end{figure}

\section{Signature Inversion}\label{inversion}

In this section, we provide explicit signature inversion formulae. We do so by expressing the coefficients of the expansion of a path in the Fourier or orthogonal polynomial bases as a linear function on the signature. The necessary background material on orthogonal polynomials and Fourier series can be found in \cref{appendix:orthogonal_polynomials}.

Throughout this section, $x:[0,1] \to \mathbb{R}$ will denote a $1$-dimensional smooth path. The results in the sequel can be naturally extended to multidimensional paths by applying the same procedure channel by channel. Depending on the type of basis we chose to represent the path, we will often need to reparameterise the path from the interval $[0,1]$ to a specified time interval $[a,b]$ and augment it with time as well as with additional channels $c^1, c^2, ..., c^r : [a,b]\to\mathbb{R}$, tailor-made for the specific type of inversion. We denote the augmented path by $\hat{x}(t)=(t, c^1(t), ..., c^r(t), x(t)) \in \mathbb R^{r+2}$. Note that these transformations are fully deterministic and do not affect the complexity of the generation task outlined in \cref{sec:diff-sig}. Furthermore, we will use the shorthand notation $S(\hat x)$ for the step-$n$ signature $S^{\leq n}(\hat x)$ throughout the section, and assume that the truncation level $n$ is always high enough to retrieve the desired number of basis coefficients.  All proofs can be found in \cref{sec:proofs}.

\subsection{Inversion via Fourier coefficients}\label{subsection:fourier_inversion}

In this section, we derive closed-form expressions for retrieving the first $n$ Fourier coefficients of a path from its signature. First, recall that the Fourier series of a $2\pi$-periodic path $x(t)$ up to order $n \in \mathbb N$ is $x_n(t) = a_0 + \sum_{n=1}^n (a_n \cos (nt) + b_n \sin (nt))$
where $a_0,a_n,b_n$ are defined as
\begin{equation}\label{eqn: a0}
    a_0 = \frac{1}{2\pi}\int_0^{2\pi}x(t)dt,
\end{equation}
\begin{equation}\label{eqn: an}
    a_n = \frac{1}{\pi}\int_0^{2\pi}x(t)\cos(nt)dt,
\end{equation}
\begin{equation}\label{eqn: bn}
    b_n = \frac{1}{\pi}\int_0^{2\pi}x(t)\sin(nt)dt.
\end{equation}

\begin{theorem}\label{thm:fourier_inversion} Let $x : [0,2\pi] \to \mathbb R$ be a periodic smooth path such that $x(0)=0$, and consider the augmentation $\hat{x}(t) = (t, \sin(t), \cos(t)-1, x(t)) \in \mathbb R^4$. Then the following relations hold
\begin{equation}
    \begin{split}
    a_0 &= \frac{1}{2\pi}\langle e_4 \succ e_1, S(\hat{x}) \rangle, \\
    a_n &= \frac{1}{\pi}\sum_{k=0}^n \sum_{q=0}^k \binom{n}{k}\binom{k}{q} \cos(\frac{1}{2}(n-k)\pi) \langle (e_4\shuffle e_2^{\shuffle n-k} \shuffle e_3^{\shuffle q} ) \succ e_1, S(\hat{x}) \rangle,\\
    b_n &= \frac{1}{\pi}\sum_{k=0}^n \sum_{q=0}^k \binom{n}{k}\binom{k}{q} \sin(\frac{1}{2}(n-k)\pi) \langle (e_4\shuffle e_2^{\shuffle n-k} \shuffle e_3^{\shuffle q} ) \succ e_1, S(\hat{x})\rangle.
    \end{split}
\end{equation}
\end{theorem}

\subsection{Inversion via orthogonal polynomials}
To accommodate path generation use cases for which a non-Fourier representation is more suitable, next we derive formulae for inverting the signature using expansions of the path in orthogonal polynomial bases.
Recall that any orthogonal polynomial family $(p_n)_{n\in\mathbb{N}}$ with a weight function $\omega: [a, b]\rightarrow \mathbb{R}$ satisfies a three-term recurrence relation
\begin{align}\label{eq:general_rec}
    p_{n}(t) = (A_{n}t + B_{n})p_{n-1}(t) + C_{n}p_{n-2}(t), \qquad n\geq2,
    \end{align}
with $p_0(t) = 1$ and $p_1(t) = A_1 t +B_1$. 
Also, note that any smooth (or at least square-integrable) path $x(t)$ with $x(a)=0$ can be approximated arbitrarily well as  $x(t)\approx\sum_{n=0}^\infty \alpha_np_n(t)$ where $\alpha_n$ is the $n$-th orthogonal polynomial coefficient
\begin{equation}\label{eqn:coeff-poly}
    \alpha_n = \frac{1}{( p_n, p_n)} \int_{a}^{b}x(t)p_n(t)\omega(t)dt,
\end{equation}
and $(\cdot, \cdot)$ denotes the inner product $( f, g)=\int_{a}^{b}f(t)g(t)\omega(t)dt$.
We include several examples of such polynomial families in \cref{appendix:orthogonal_polynomials}.

\begin{theorem}\label{thm:poly_inversion}
Let $x : [a,b] \to \mathbb R$ be a smooth path such that  $x(a)=0$. Consider the augmentation $\hat x(t) = (t, \omega(t)x(t)) \in \mathbb R^2$, where $\omega(t)$ corresponds to the weight function of a system of orthogonal polynomials $(p_n)_{n\in\mathbb{N}}$ and is well defined on the closed and compact interval $[a,b]$. Then there exists a linear combination $\ell_n$ of words such that the $n^{th}$ coefficient in Equation (\ref{eqn:coeff-poly}) satisfies $\alpha_n=\langle\ell_n, S(\hat x)\rangle$. Furthermore, the sequence $(\ell_n)_{n \in \mathbb N}$ satisfies the following recurrence relation
\begin{align*}
    \ell_{n} = A_n\frac{(p_{n-1}, p_{n-1})}{(p_n, p_n)}e_1\succ \ell_{n-1} + (A_na+B_n)\frac{(p_{n-1}, p_{n-1})}{(p_n, p_n)}\ell_{n-1} + C_n\frac{(p_{n-2}, p_{n-2})}{(p_n, p_n)}\ell_{n-2},
\end{align*}
with 
\begin{align*}
    \ell_0=\frac {A_0}{(p_0, p_0)}e_{21}\quad \text{ and }\quad\ell_1=\frac {A_1}{(p_1, p_1)}(e_{121} + e_{211})+\frac {A_1a+B_1}{(p_1, p_1)}e_{21}.
\end{align*}
\end{theorem}

\begin{remark}\label{remark: taylor}
    The results in \cref{thm:poly_inversion} require signatures of $\hat x = (t, w(t)x(t))$. However, sometimes one may only have signatures of $\tilde x = (t, x(t))$. In \cref{appendix:taylor} we propose an alternative method by approximating the weight function as a Taylor series. 
\end{remark}

\subsection{Inversion time complexity}
The inversion formulae all boil down to evaluating specific linear combinations of signature terms. Evaluating a linear functional has a time complexity linear in the size of the signature. Since this evaluation is repeated for each of the $n$ recovered basis coefficients, the total number of operations is $nm$, where $n$ in the number of basis coefficients and $m$ is the length of the signature truncated at level $n+2$. For a $d$-dimensional path, length of a step-$N$ signature is $\frac{d^{N+1}-1}{d-1}$, giving the inversion a time complexity of $O(nd^{n+2})$.

\section{Related Work}\label{sec:related-work}

\paragraph{Multivariate time series generation} Generating multivariate time series has been an active area of research in the past several years, predominantly relying on generative adversarial networks (GANs) \citep{goodfellow2014generative}. Simple recurrent neural networks acting as generators and discriminators \citep{mogren2016c, esteban2017real} later evolved into encoder-decoder architectures where the adversarial generation occurs in a learned latent space \citep{yoon2019time, pei2021towards, jeon2022gt}. To generate time series in continuous time, architectures based on neural differential equations in the latent space \citep{rubanova2019latent, yildiz2019ode2vae} have emerged as generalisations of RNNs. 

More flexible alternatives have since been proposed in the forms of neural controlled differential equations \citep{kidger2020neural} and state space ODEs \citep{zhou2023deep}.
 
\paragraph{Diffusion models for time series generation} There are a number of denoising probabilistic diffusion models (DDPMs) currently at the forefront of time series synthesis, such as DiffTime \citep{ho2020denoising}, which reformulates the constrained time series generation problem in terms of conditional denoising diffusion \citep{tashiro2021csdi}. Most recently, Diffusion-TS \citep{yuan2024diffusion} has demonstrated superior performance on benchmark datasets and long time series by disentangling temporal features via a Fourier-based training objective. To learn long-range dependencies, both the aforementioned methods use transformer-based \citep{vaswani2017attention} diffusion functions. Many recent efforts attempt to generalise score-based diffusion to infinite-dimensional function spaces \citep{kerrigan2022diffusion, dutordoir2023neural, phillips2022spectral, lim2023score}. However, unlike their discrete-time counterparts, they have not yet been benchmarked on a variety of real-world temporal datasets. An exception to this is a diffusion framework proposed by \citet{bilovs2023modeling}, which synthesises continuous time series by replacing the time-independent noise corruption with samples from a Gaussian process, forcing the diffusion to remain in the space of continuous functions. Another promising approach for training diffusion models in function space is the Denoising Diffusion Operators (DDOs) method \citep{lim2023score}. While it has not been previously applied to time series, our evaluation in \cref{experiments} demonstrates its strong performance in this context. Additionally, there is a growing body of recent literature focusing on application-specific time series generation via diffusion models, such as speech enhancement \citep{lay2023reducing, lemercier2023storm}, soft sensing \citep{dai2023timeddpm}, and battery charging behaviour \citep{li2024diffcharge}.

\paragraph{Signature inversion} The uniqueness property of signatures mentioned in \cref{section:sig-properties} has motivated several previous attempts to answer the question of inverting the signature transform, mostly as theoretical contributions focusing on one specific class of paths \citep{lyons2017hyperbolic, chang2016signature, lyons2018inverting}. The only fast and scalable signature inversion strategy to date is the Insertion method \citep{chang2019insertion}, which provides an algorithm and theoretical error bounds for inverting piecewise linear paths. It was recently optimised \citep{fermanian2023insertion} and released as a part of the Signatory \citep{kidger2019deep} package. There are also examples of inversion via deep learning \citep{kidger2019deep} and evolutionary algorithms \citep{buehler2020data}, but they provide no convergence guarantees and become largely inefficient when deployed on real-world time series.

\section{Experiments}\label{experiments}
In \cref{subsection:inv_eval}, we demonstrate that the newly proposed signature inversion method achieves more accurate reconstructions than the previous Insertion \citep{chang2019insertion, fermanian2023insertion} and Optimisation \citep{kidger2019deep} methods. We also analyse the inversion quality and time complexity across different families of orthogonal polynomials.
In \cref{subsection:smooth_eval}, we show that (log)signatures, combined with our closed-form inversion, provide an exceptionally effective embedding for time series diffusion models. \cref{section:tradeoff} visualises the trade-off between the precision of time series representation and model complexity introduced by the choice of the signature truncation level. In \cref{section:step-4}, we present experiments demonstrating that generating step-$4$ log-signatures via the pipeline outlined in \cref{fig:sigdiffusions_diagram} outperforms other recent time series diffusion models across several standard metrics.
\subsection{Inversion evaluation}\label{subsection:inv_eval}

We perform experiments to evaluate the proposed analytical signature inversion formulae derived in \cref{inversion} via several families of orthogonal bases. Using example paths given by sums of random sine waves with injected Gaussian noise, we reconstruct the original paths from their step-$12$ signatures. Figure \ref{fig:compare_inversion} compares inversion of these paths via Legendre and Fourier coefficients to the Insertion method \citep{chang2019insertion, fermanian2023insertion} and Optimisation method \citep{kidger2019deep}, showcasing the improvement in inversion quality provided by our explicit inversion formulae. 
\begin{figure}
    \centering
    \includegraphics[width=0.9\linewidth]{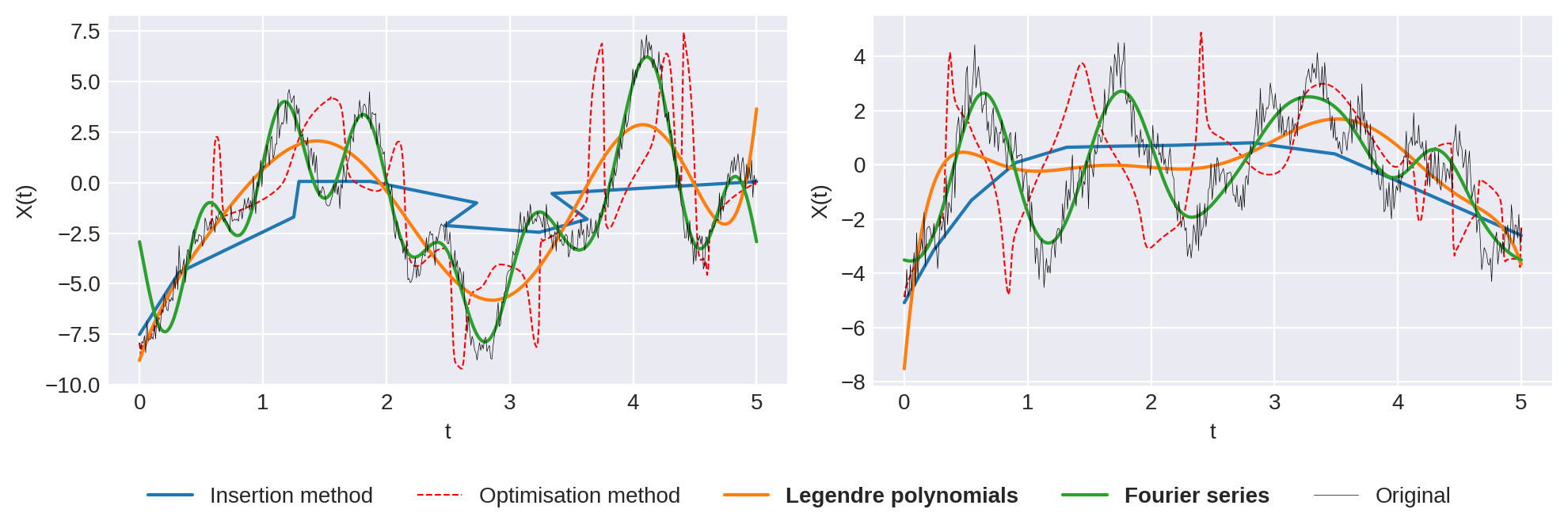}
    \caption{Comparison of different inversion methods.}
    \label{fig:compare_inversion}
\end{figure}
Figure \ref{fig:inv_error} presents the time consumption against the $L_2$ error as the degree of polynomials increases.
\begin{figure}[htbp]
    \centering
    \includegraphics[width=1\linewidth]{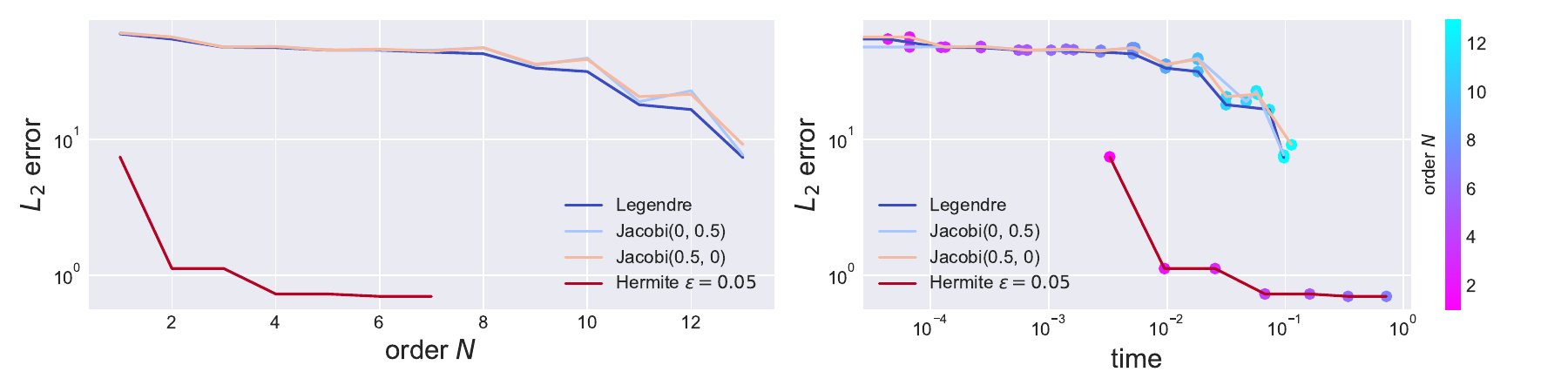}
    \caption{$L_2$ error of signature inversion via orthogonal polynomials with respect to the polynomial order $N$ and time. Error and time are calculated by an average over 15 paths with $200$ sample points.}
    \label{fig:inv_error}
\end{figure} 
Notably, the factor holding the most influence over the reconstruction quality is the truncation level of the signature, as it bounds the order of polynomials we can retrieve. We refer the interested reader to a discussion about inversion quality in \cref{appendix:visualisation}. Namely, Figures \ref{fig:visual_sig_inv_rough} and \ref{fig:visual_sig_inv_rough_real} show more examples of inverted signatures using different types of paths and polynomial bases.

\subsection{Generating long time series}\label{subsection:smooth_eval}
In this section, we introduce the \texttt{SigDiffusions} pipeline for generating multivariate time series by the following strategy, also illustrated in \cref{fig:sigdiffusions_diagram}: 
\begin{enumerate}
    \item Choose an orthogonal basis and order $N$ sufficient to represent the time series with enough detail to retain its meaningful components while smoothing out unnecessary noise.
    \item Compute the log-signatures of the time series, augmented as described in \cref{inversion} and truncated at level $N+2$.
    \item Train and sample a score-based diffusion model to generate the log-signatures.
    \item Invert the synthetic log-signatures back to time series using our formulae from \cref{inversion}.
    
\end{enumerate}

\subsubsection{Time series representation and model capacity}\label{section:tradeoff}
Here, we highlight the trade-off between model capacity and the faithfulness of the time series representation as given by its truncated Fourier series.
Since the signature inversion formulae are exact, the inversion quality depends only on how well the underlying signal is approximated by the retrieved Fourier basis coefficients. When the primary goal is to model the overall shape of a multivariate time series, it is sufficient to truncate the signature, and thus the Fourier series, at a low level. Low truncation levels capture most of this information, smooth out high-order noise, and simplify the generation task.
In contrast, modelling signals with high-frequency components requires higher signature truncation levels, raising the complexity of the diffusion task. As the signature size grows exponentially with the truncation level, generating highly oscillatory time series with high fidelity becomes increasingly constrained by the model's capacity.
To illustrate this, we show the generated samples from \texttt{SigDiffusions} trained on a noisy Lotka–Volterra system with different log-signature truncation levels in Figure \ref{fig:samples_lvls}. At each truncation level, the log-signatures become progressively more difficult to generate accurately.

\begin{figure}[hbtp]
    \centering
    \includegraphics[width=\linewidth]{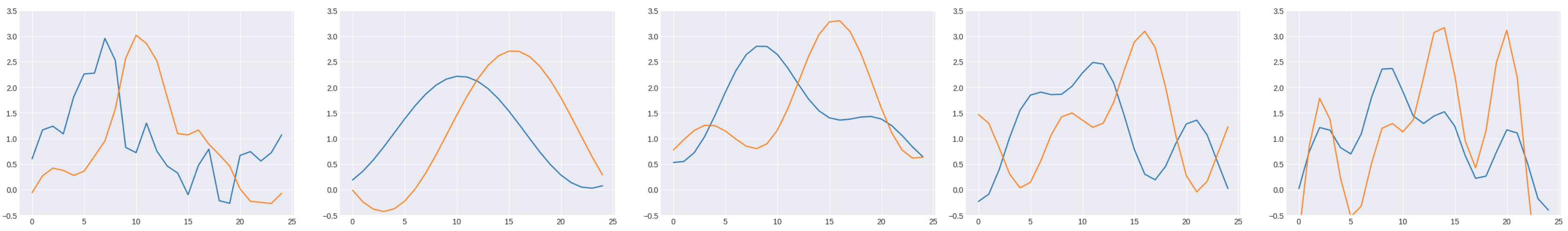}
    \caption{Time series representation and model capacity trade-off. Left to right: real sample from a noisy Lotka–Volterra system, sample generated by \texttt{SigDiffusion} with signature truncation level 3, 4, 5, 6.}
    \label{fig:samples_lvls}
\end{figure}

\subsubsection{Generating step-$4$ log-signatures}\label{section:step-4}
We now demonstrate the exceptional ability of signatures to capture the shape and cross-channel dependencies of time series at low truncation levels. We show that \texttt{SigDiffusions} applied to step-$4$ log-signatures, combined with Fourier inversion, outperform state-of-the-art diffusion-based models on the task of generating 1000-point-long time series.

\paragraph{Datasets}
We perform experiments on five different time series datasets: \textbf{Sines} \citep{yoon2019time} - a dataset of 5-dimensional sine curves with randomly sampled frequency and phase, \textbf{Predator-prey} \citep{bilovs2023modeling} - a two-dimensional continuous system evolving according to a set of ODEs, \textbf{Household Electric Power Consumption (HEPC)} \citep{uciml2024electric} - the univariate \textit{voltage} feature from a real-world dataset of household power consumption, \textbf{Exchange Rates} \citep{lai2018modeling, 2017laiguokun} - a real-world dataset containing daily exchange rates of 8 currencies, and \textbf{Weather} \citep{weather} - a real-world dataset reporting weather measurements.

\paragraph{Metrics}
We use the metrics established in \citet{yoon2019time}. The \textbf{Discriminative Score} reports the out-of-sample accuracy of a recurrent neural network (RNN) classifier trained to distinguish between real and generated time series. To improve readability, we report values offset by $0.5$, so that in the ideal case, where real and generated samples are indistinguishable by the classifier, this metric will approach $0$. The \textbf{Predictive Score} measures the loss of a next-point predictor RNN trained exclusively on synthetic data, with the loss evaluated on the real test data set. In the case of univariate time series, the RNN predicts the value of the time series 20 time points ahead, instead of predicting the value of an unseen channel. We also run the Kolmogorov-Smirnov (KS) test on marginal distributions of random batches of ground truth and generated paths. We repeat this test 1000 times with a batch size of 64 and report the mean KS score with the mean Type I error for a 5\% significance threshold.
Since the cross-channel terms of the log-signature are not necessary for the inversion methods, we generate a concatenated vector of the log-signatures of each separate dimension plus the augmentation described in \cref{inversion}.

\paragraph{Benchmarks}
Table \ref{table:disc_pred} lists the time series generation performance metrics compared with four recent diffusion model architectures specifically designed to handle long or continuous-time paths: 

\begin{itemize}[leftmargin=0.5cm]
    \item \textbf{Diffusion-TS \citep{yuan2024diffusion}} This model introduces a novel Fourier-based training objective to disentangle temporal features of different seasonalities. This interpretable decomposition strategy makes the model particularly robust to varying time series lengths, demonstrating strong performance relative to other benchmarks as the training time series become longer.
    \item \textbf{CSPD-GP \citep{bilovs2023modeling}} This approach replaces the time-independent noise corruption mechanism with samples drawn from a Gaussian process, effectively modelling diffusion on time series as a process occurring within the space of continuous functions. CSPD-GP (RNN) and CSPD-GP (Transformer) refer to score-based diffusion models with the score function either being an RNN or a transformer.
    \item \textbf{Denoising Diffusion Operators (DDO) \citep{lim2023score}} These models generalise diffusion models to function spaces with a Hilbert space-valued Gaussian process to perturb the input data. Additionally, they use neural operators for the score function, ensuring consistency with the underlying function space formulation. DDO's kernel smoothness hyperparameter $\gamma$ is tuned and reported for each dataset.
\end{itemize}
 
We compute the metrics using 1000 samples from each model. Table \ref{table:model_size} shows the model sizes and training times, demonstrating that \texttt{SigDiffusion} outperforms the other models while also having the most efficient architecture. Table \ref{table: marginals} in the Appendix evaluates the time series marginals using the KS test. More details about the experimental setup can be found in \cref{appendix:experiments}.

\begin{table}
  \caption{Results for generating time series of length 1000.}
  \label{table:disc_pred}
  \centering
  \begin{tabular}{llll}
    \toprule
    Dataset & Model & Discriminative Score & Predictive Score \\
    \midrule
   \multirow{4}{*}{Sines}& SigDiffusion (ours) &  \underline{0.100$\pm$.026} & \underline{0.191$\pm$.004} \\
   & DDO ($\gamma = 1$) & \textbf{0.040$\pm$.025} & \textbf{0.187$\pm$.003} \\
   & Diffusion-TS & 0.291$\pm$.070 &  0.365$\pm$.006 \\
   & CSPD-GP (RNN) & 0.468$\pm$.030 & 0.244$\pm$.012 \\
   & CSPD-GP (Transformer) & 0.493$\pm$.003 & 0.390$\pm$.023 \\
    \midrule
    \multirow{4}{*}{Predator-prey}& SigDiffusion (ours) & \underline{0.184$\pm$.058}  & \textbf{0.056$\pm$.000}\\
   & DDO ($\gamma = 10$) & 0.211$\pm$.080 & \textbf{0.056$\pm$.000}\\  
   & Diffusion-TS & 0.500$\pm$.000 & 0.482$\pm$.023\\
   & CSPD-GP (RNN) & \textbf{0.168$\pm$.024} & 0.057$\pm$.000 \\
   & CSPD-GP (Transformer) & 0.498$\pm$.003 & 0.919$\pm$.002\\
    \midrule
    \multirow{4}{*}{HEPC}& SigDiffusion (ours) & \textbf{0.070$\pm$.032} & \underline{0.050$\pm$.012}\\
   & DDO ($\gamma = 1$) & \underline{0.081$\pm$.019} & \textbf{0.044$\pm$.001} \\
   & Diffusion-TS & 0.438$\pm$.057  & 0.066$\pm$.022\\
   & CSPD-GP (RNN) & 0.415$\pm$.045 & 0.108$\pm$.002\\
   & CSPD-GP (Transformer) & 0.500$\pm$.000 & 0.551$\pm$.028 \\
   \midrule
       \multirow{4}{*}{Exchange Rates}& SigDiffusion (ours) & \textbf{0.278$\pm$.062} & \textbf{0.057$\pm$.001} \\
    & DDO ($\gamma = 1$) & \underline{0.326$\pm$.102} & \underline{0.094$\pm$.004}\\
   & Diffusion-TS & 0.401$\pm$.196 & 0.120$\pm$.016 \\
   & CSPD-GP (RNN) & 0.500$\pm$.001 & 0.273$\pm$.100 \\
   & CSPD-GP (Transformer) & 0.500$\pm$.000 & 0.432$\pm$.074 \\
   \midrule
    \multirow{4}{*}{Weather}& SigDiffusion (ours) & \textbf{0.350$\pm$.080} & \textbf{0.168$\pm$.001} \\
   & DDO ($\gamma = 10$) & \underline{0.356$\pm$.196} & \underline{0.307$\pm$.007} \\
   & Diffusion-TS & 0.498$\pm$.003 & 0.438$\pm$.035 \\
   & CSPD-GP (RNN) & 0.500$\pm$.000 & 0.505$\pm$.007 \\
   & CSPD-GP (Transformer) &  0.500$\pm$.000 & 0.490$\pm$.000 \\
   \bottomrule
  \end{tabular}
\end{table}

\begin{table}[ht]
  \caption{Comparison of model sizes.}
  \label{table:model_size}
  \centering
  \begin{tabular}{lllll}
    \toprule
    Dataset & Model & Parameters & Training Time & Sampling Time\\
    \midrule
   \multirow{4}{*}{Sines} & SigDiffusion (ours) & 229K & 8 min & 11 sec\\
   & DDO ($\gamma = 1$) & 4.12M & 3.6 h & 42 min \\
   & Diffusion-TS & 4.18M & 57 min & 15 min \\
   & CSPD-GP (RNN) & 759K & 9 min & 1 min \\
   & CSPD-GP (Transformer) & 973K & 15 min & 5 min \\
    \midrule
   \multirow{4}{*}{Predator-prey} & SigDiffusion (ours) & 211K & 8 min  & 12 sec\\
   & DDO ($\gamma = 10$) & 4.12M & 3.5 h & 42 min \\
   & Diffusion-TS & 4.17M & 55 min & 14 min  \\
   & CSPD-GP (RNN) & 758K & 8 min & 1 min \\
   & CSPD-GP (Transformer) & 972K & 16 min  & 5 min  \\
    \midrule
   \multirow{4}{*}{HEPC} & SigDiffusion (ours) & 206K & 8 min & 12 sec \\
   & DDO ($\gamma = 1$) & 4.12M & 2.6 h & 42 min \\
   & Diffusion-TS & 4.17M  & 50 min & 15 min \\
   & CSPD-GP (RNN) & 758K & 4 min  & 1 min \\
   & CSPD-GP (Transformer) & 972K & 9 min  & 5 min \\
    \midrule
    \multirow{4}{*}{Exchange rates} & SigDiffusion (ours) & 247 K & 9 min & 11 sec \\
    & DDO ($\gamma = 1$) & 4.12M & 3.6 h & 42 min \\
   & Diffusion-TS & 4.29 M & 1.2 h & 20 min \\
   & CSPD-GP (RNN) & 760 K & 11 min & 1 min \\
   & CSPD-GP (Transformer) & 974 K & 15 min & 6 min \\
    \midrule
    \multirow{4}{*}{Weather} &  SigDiffusion (ours) & 282K & 8 min & 12 sec \\
    & DDO ($\gamma = 10$) & 4.12M & 3.6 h & 42 min \\
   & Diffusion-TS & 4.3 M & 1.4 h & 20 min \\
   & CSPD-GP (RNN) & 763 K & 14 min & 1 min \\
   & CSPD-GP (Transformer) & 975 K & 17 min & 6 min \\
    \bottomrule
  \end{tabular}
\end{table}

\section{Conclusion and Future Work}\label{conclusion}
In this paper, we introduced \texttt{SigDiffusion}, a new diffusion model that gradually perturbs and denoises log-signature embeddings of long time series, preserving their Lie algebraic structure. To recover the path from its log-signature, we proved that the coefficients of the expansion of a path in a given basis, such as Fourier or orthogonal polynomials, can be expressed as explicit linear functionals on the signature, or equivalently as polynomial functions on the log-signature. These results provide explicit signature inversion formulae, representing a major improvement over signature inversion algorithms previously proposed in the literature. Finally, we demonstrated how combining \texttt{SigDiffusion} with these inversion formulae provides a powerful generative approach for time series that is competitive with state-of-the-art diffusion models for temporal data.

As this is the first work on diffusion models for time series using signature embeddings, there are still many research directions to explore. To mitigate the rapid growth in the number of required signature features for high-frequency signals described in \cref{section:tradeoff}, future work could explore alternative embeddings to the signature. For example, other types of path developments derived from rough path theory, which embed temporal signals into (compact) Lie groups, such as those proposed by \citet{cass2024free}, may offer a more parsimonious representation. These alternatives retain many of the desirable properties of signatures, including the availability of a flat-space Lie algebra. However, it remains unclear how an inversion mechanism would work in these cases. Finally, it would be interesting to understand how \emph{discrete-time signatures} \citep{diehl2023generalized} could be leveraged to encode \emph{discrete sequences} on Lie groups and potentially perform diffusion-based generative modelling for text.

\subsubsection*{Acknowledgments}
Barbora Barancikova is supported by UK Research and Innovation [UKRI Centre for Doctoral Training in AI for Healthcare
grant number EP/S023283/1].

\bibliography{iclr2025_conference}

\begin{thebibliography}{79}
\providecommand{\natexlab}[1]{#1}
\providecommand{\url}[1]{\texttt{#1}}
\expandafter\ifx\csname urlstyle\endcsname\relax
  \providecommand{\doi}[1]{doi: #1}\else
  \providecommand{\doi}{doi: \begingroup \urlstyle{rm}\Url}\fi

\bibitem[Anderson(1982)]{anderson1982reverse}
Brian~DO Anderson.
\newblock Reverse-time diffusion equation models.
\newblock \emph{Stochastic Processes and their Applications}, 12\penalty0 (3):\penalty0 313--326, 1982.

\bibitem[Arribas et~al.(2020)Arribas, Salvi, and Szpruch]{arribas2020sig}
Imanol~Perez Arribas, Cristopher Salvi, and Lukasz Szpruch.
\newblock Sig-sdes model for quantitative finance.
\newblock In \emph{Proceedings of the First ACM International Conference on AI in Finance}, pp.\  1--8, 2020.

\bibitem[Atkinson(2009)]{the_num_ana}
Kendall. Atkinson.
\newblock \emph{Theoretical Numerical Analysis A Functional Analysis Framework}.
\newblock Texts in Applied Mathematics, 39. Springer New York, New York, NY, 3rd ed. 2009. edition, 2009.
\newblock ISBN 1-282-33318-6.

\bibitem[Bilo{\v{s}} et~al.(2023)Bilo{\v{s}}, Rasul, Schneider, Nevmyvaka, and G{\"u}nnemann]{bilovs2023modeling}
Marin Bilo{\v{s}}, Kashif Rasul, Anderson Schneider, Yuriy Nevmyvaka, and Stephan G{\"u}nnemann.
\newblock Modeling temporal data as continuous functions with stochastic process diffusion.
\newblock In \emph{International Conference on Machine Learning}, pp.\  2452--2470. PMLR, 2023.

\bibitem[Buehler et~al.(2020)Buehler, Horvath, Lyons, Arribas, and Wood]{buehler2020data}
Hans Buehler, Blanka Horvath, Terry Lyons, Imanol~Perez Arribas, and Ben Wood.
\newblock A data-driven market simulator for small data environments.
\newblock \emph{arXiv preprint arXiv:2006.14498}, 2020.

\bibitem[Cai et~al.(2020)Cai, Yang, Averbuch-Elor, Hao, Belongie, Snavely, and Hariharan]{cai2020learning}
Ruojin Cai, Guandao Yang, Hadar Averbuch-Elor, Zekun Hao, Serge Belongie, Noah Snavely, and Bharath Hariharan.
\newblock Learning gradient fields for shape generation.
\newblock In \emph{Computer Vision--ECCV 2020: 16th European Conference, Glasgow, UK, August 23--28, 2020, Proceedings, Part III 16}, pp.\  364--381. Springer, 2020.

\bibitem[Cass \& Salvi(2024)Cass and Salvi]{cass2024lecture}
Thomas Cass and Cristopher Salvi.
\newblock Lecture notes on rough paths and applications to machine learning.
\newblock \emph{arXiv preprint arXiv:2404.06583}, 2024.

\bibitem[Cass \& Turner(2024)Cass and Turner]{cass2024free}
Thomas Cass and William~F Turner.
\newblock Free probability, path developments and signature kernels as universal scaling limits.
\newblock \emph{arXiv preprint arXiv:2402.12311}, 2024.

\bibitem[Chang \& Lyons(2019)Chang and Lyons]{chang2019insertion}
Jiawei Chang and Terry Lyons.
\newblock Insertion algorithm for inverting the signature of a path.
\newblock \emph{arXiv preprint arXiv:1907.08423}, 2019.

\bibitem[Chang et~al.(2016)Chang, Duffield, Ni, and Xu]{chang2016signature}
Jiawei Chang, Nick Duffield, Hao Ni, and Weijun Xu.
\newblock Signature inversion for monotone paths, 2016.

\bibitem[Chen(1957)]{chen1957integration}
Kuo-Tsai Chen.
\newblock Integration of paths, geometric invariants and a generalized baker-hausdorff formula.
\newblock \emph{Annals of Mathematics}, 65\penalty0 (1):\penalty0 163--178, 1957.

\bibitem[Chen(1958)]{chen1958integration}
Kuo-Tsai Chen.
\newblock Integration of paths--a faithful representation of paths by noncommutative formal power series.
\newblock \emph{Transactions of the American Mathematical Society}, 89\penalty0 (2):\penalty0 395--407, 1958.

\bibitem[Chow(1939)]{chow1939system}
WL~Chow.
\newblock On system of linear partial differential equations of the first order.
\newblock \emph{Mathematische Annalen}, 117\penalty0 (1):\penalty0 98--105, 1939.

\bibitem[Cirone et~al.(2023)Cirone, Lemercier, and Salvi]{cirone2023neural}
Nicola~Muca Cirone, Maud Lemercier, and Cristopher Salvi.
\newblock Neural signature kernels as infinite-width-depth-limits of controlled resnets.
\newblock In \emph{International Conference on Machine Learning}, pp.\  25358--25425. PMLR, 2023.

\bibitem[Cirone et~al.(2024)Cirone, Orvieto, Walker, Salvi, and Lyons]{cirone2024theoretical}
Nicola~Muca Cirone, Antonio Orvieto, Benjamin Walker, Cristopher Salvi, and Terry Lyons.
\newblock Theoretical foundations of deep selective state-space models.
\newblock \emph{arXiv preprint arXiv:2402.19047}, 2024.

\bibitem[Cochrane et~al.(2021)Cochrane, Foster, Chhabra, Lemercier, Lyons, and Salvi]{cochrane2021sk}
Thomas Cochrane, Peter Foster, Varun Chhabra, Maud Lemercier, Terry Lyons, and Cristopher Salvi.
\newblock Sk-tree: a systematic malware detection algorithm on streaming trees via the signature kernel.
\newblock In \emph{2021 IEEE international conference on cyber security and resilience (CSR)}, pp.\  35--40. IEEE, 2021.

\bibitem[Coletta et~al.(2024)Coletta, Gopalakrishnan, Borrajo, and Vyetrenko]{coletta2024constrained}
Andrea Coletta, Sriram Gopalakrishnan, Daniel Borrajo, and Svitlana Vyetrenko.
\newblock On the constrained time-series generation problem.
\newblock \emph{Advances in Neural Information Processing Systems}, 36, 2024.

\bibitem[Dai et~al.(2023)Dai, Yang, Liu, Liu, and Liu]{dai2023timeddpm}
Yun Dai, Chao Yang, Kaixin Liu, Angpeng Liu, and Yi~Liu.
\newblock Timeddpm: Time series augmentation strategy for industrial soft sensing.
\newblock \emph{IEEE Sensors Journal}, 2023.

\bibitem[Diehl et~al.(2023)Diehl, Ebrahimi-Fard, and Tapia]{diehl2023generalized}
Joscha Diehl, Kurusch Ebrahimi-Fard, and Nikolas Tapia.
\newblock Generalized iterated-sums signatures.
\newblock \emph{Journal of Algebra}, 632:\penalty0 801--824, 2023.

\bibitem[Dutordoir et~al.(2023)Dutordoir, Saul, Ghahramani, and Simpson]{dutordoir2023neural}
Vincent Dutordoir, Alan Saul, Zoubin Ghahramani, and Fergus Simpson.
\newblock Neural diffusion processes.
\newblock In \emph{International Conference on Machine Learning}, pp.\  8990--9012. PMLR, 2023.

\bibitem[Esteban et~al.(2017)Esteban, Hyland, and R{\"a}tsch]{esteban2017real}
Crist{\'o}bal Esteban, Stephanie~L Hyland, and Gunnar R{\"a}tsch.
\newblock Real-valued (medical) time series generation with recurrent conditional gans.
\newblock \emph{arXiv preprint arXiv:1706.02633}, 2017.

\bibitem[Fermanian et~al.(2023{\natexlab{a}})Fermanian, Chang, Lyons, and Biau]{fermanian2023insertion}
Adeline Fermanian, Jiawei Chang, Terry Lyons, and G{\'e}rard Biau.
\newblock The insertion method to invert the signature of a path.
\newblock \emph{arXiv preprint arXiv:2304.01862}, 2023{\natexlab{a}}.

\bibitem[Fermanian et~al.(2023{\natexlab{b}})Fermanian, Lyons, Morrill, and Salvi]{fermanian2023new}
Adeline Fermanian, Terry Lyons, James Morrill, and Cristopher Salvi.
\newblock New directions in the applications of rough path theory.
\newblock \emph{IEEE BITS the Information Theory Magazine}, 2023{\natexlab{b}}.

\bibitem[Fliess et~al.(1983)Fliess, Lamnabhi, and Lamnabhi-Lagarrigue]{fliess1983algebraic}
Michel Fliess, Moustanir Lamnabhi, and Fran{\c{c}}oise Lamnabhi-Lagarrigue.
\newblock An algebraic approach to nonlinear functional expansions.
\newblock \emph{IEEE transactions on circuits and systems}, 30\penalty0 (8):\penalty0 554--570, 1983.

\bibitem[Friz \& Victoir(2010)Friz and Victoir]{friz2010multidimensional}
Peter~K Friz and Nicolas~B Victoir.
\newblock \emph{Multidimensional stochastic processes as rough paths: theory and applications}, volume 120.
\newblock Cambridge University Press, 2010.

\bibitem[Goodfellow et~al.(2014)Goodfellow, Pouget-Abadie, Mirza, Xu, Warde-Farley, Ozair, Courville, and Bengio]{goodfellow2014generative}
Ian Goodfellow, Jean Pouget-Abadie, Mehdi Mirza, Bing Xu, David Warde-Farley, Sherjil Ozair, Aaron Courville, and Yoshua Bengio.
\newblock Generative adversarial nets.
\newblock \emph{Advances in neural information processing systems}, 27, 2014.

\bibitem[Ho et~al.(2020)Ho, Jain, and Abbeel]{ho2020denoising}
Jonathan Ho, Ajay Jain, and Pieter Abbeel.
\newblock Denoising diffusion probabilistic models.
\newblock \emph{Advances in neural information processing systems}, 33:\penalty0 6840--6851, 2020.

\bibitem[Horvath et~al.(2023)Horvath, Lemercier, Liu, Lyons, and Salvi]{horvath2023optimal}
Blanka Horvath, Maud Lemercier, Chong Liu, Terry Lyons, and Cristopher Salvi.
\newblock Optimal stopping via distribution regression: a higher rank signature approach.
\newblock \emph{arXiv preprint arXiv:2304.01479}, 2023.

\bibitem[Hwang et~al.(2023)Hwang, Lim, Lee, and Choi]{hwang2023augmentation}
Yechan Hwang, Jinsu Lim, Young-Jun Lee, and Ho-Jin Choi.
\newblock Augmentation for context in financial numerical reasoning over textual and tabular data with large-scale language model.
\newblock In \emph{NeurIPS 2023 Second Table Representation Learning Workshop}, 2023.

\bibitem[Ismail(2005)]{quantum_ortho}
Mourad Ismail.
\newblock \emph{Classical and quantum orthogonal polynomials in one variable /}.
\newblock Encyclopedia of mathematics and its applications ; v. 98. Cambridge University Press, Cambridge, 2005.
\newblock ISBN 9780521782012.

\bibitem[Issa et~al.(2024)Issa, Horvath, Lemercier, and Salvi]{issa2024non}
Zacharia Issa, Blanka Horvath, Maud Lemercier, and Cristopher Salvi.
\newblock Non-adversarial training of neural sdes with signature kernel scores.
\newblock \emph{Advances in Neural Information Processing Systems}, 36, 2024.

\bibitem[Jeon et~al.(2022)Jeon, Kim, Song, Cho, and Park]{jeon2022gt}
Jinsung Jeon, Jeonghak Kim, Haryong Song, Seunghyeon Cho, and Noseong Park.
\newblock Gt-gan: General purpose time series synthesis with generative adversarial networks.
\newblock \emph{Advances in Neural Information Processing Systems}, 35:\penalty0 36999--37010, 2022.

\bibitem[Katharopoulos et~al.(2020)Katharopoulos, Vyas, Pappas, and Fleuret]{katharopoulos2020transformers}
Angelos Katharopoulos, Apoorv Vyas, Nikolaos Pappas, and Fran{\c{c}}ois Fleuret.
\newblock Transformers are rnns: Fast autoregressive transformers with linear attention.
\newblock In \emph{International conference on machine learning}, pp.\  5156--5165. PMLR, 2020.

\bibitem[Kerrigan et~al.(2022)Kerrigan, Ley, and Smyth]{kerrigan2022diffusion}
Gavin Kerrigan, Justin Ley, and Padhraic Smyth.
\newblock Diffusion generative models in infinite dimensions.
\newblock \emph{arXiv preprint arXiv:2212.00886}, 2022.

\bibitem[Kidger et~al.(2019)Kidger, Bonnier, Perez~Arribas, Salvi, and Lyons]{kidger2019deep}
Patrick Kidger, Patric Bonnier, Imanol Perez~Arribas, Cristopher Salvi, and Terry Lyons.
\newblock Deep signature transforms.
\newblock \emph{Advances in Neural Information Processing Systems}, 32, 2019.

\bibitem[Kidger et~al.(2020)Kidger, Morrill, Foster, and Lyons]{kidger2020neural}
Patrick Kidger, James Morrill, James Foster, and Terry Lyons.
\newblock Neural controlled differential equations for irregular time series.
\newblock \emph{Advances in Neural Information Processing Systems}, 33:\penalty0 6696--6707, 2020.

\bibitem[Kolle(2024)]{weather}
Olaf Kolle.
\newblock Documentation of the weather station on top of the roof of the institute building of the max-planck-institute for biogeochemistry.
\newblock \url{https://www.bgc-jena.mpg.de/wetter/}, 2024.
\newblock Accessed: 2024-08-24.

\bibitem[Lai(2017)]{2017laiguokun}
Guokun Lai.
\newblock Multivariate time series data sets.
\newblock \url{https://github.com/laiguokun/multivariate-time-series-data}, 2017.
\newblock Accessed: 2024-08-24.

\bibitem[Lai et~al.(2018)Lai, Chang, Yang, and Liu]{lai2018modeling}
Guokun Lai, Wei-Cheng Chang, Yiming Yang, and Hanxiao Liu.
\newblock Modeling long-and short-term temporal patterns with deep neural networks.
\newblock In \emph{The 41st international ACM SIGIR conference on research \& development in information retrieval}, pp.\  95--104, 2018.

\bibitem[Lay et~al.(2023)Lay, Welker, Richter, and Gerkmann]{lay2023reducing}
Bunlong Lay, Simon Welker, Julius Richter, and Timo Gerkmann.
\newblock Reducing the prior mismatch of stochastic differential equations for diffusion-based speech enhancement.
\newblock \emph{arXiv preprint arXiv:2302.14748}, 2023.

\bibitem[Lemercier et~al.(2023)Lemercier, Richter, Welker, and Gerkmann]{lemercier2023storm}
Jean-Marie Lemercier, Julius Richter, Simon Welker, and Timo Gerkmann.
\newblock Storm: A diffusion-based stochastic regeneration model for speech enhancement and dereverberation.
\newblock \emph{IEEE/ACM Transactions on Audio, Speech, and Language Processing}, 2023.

\bibitem[Lemercier et~al.(2021{\natexlab{a}})Lemercier, Salvi, Cass, Bonilla, Damoulas, and Lyons]{lemercier2021siggpde}
Maud Lemercier, Cristopher Salvi, Thomas Cass, Edwin~V Bonilla, Theodoros Damoulas, and Terry~J Lyons.
\newblock Siggpde: Scaling sparse gaussian processes on sequential data.
\newblock In \emph{International Conference on Machine Learning}, pp.\  6233--6242. PMLR, 2021{\natexlab{a}}.

\bibitem[Lemercier et~al.(2021{\natexlab{b}})Lemercier, Salvi, Damoulas, Bonilla, and Lyons]{lemercier2021distribution}
Maud Lemercier, Cristopher Salvi, Theodoros Damoulas, Edwin Bonilla, and Terry Lyons.
\newblock Distribution regression for sequential data.
\newblock In \emph{International Conference on Artificial Intelligence and Statistics}, pp.\  3754--3762. PMLR, 2021{\natexlab{b}}.

\bibitem[Li et~al.(2024)Li, Xiong, and Chen]{li2024diffcharge}
Siyang Li, Hui Xiong, and Yize Chen.
\newblock Diffcharge: Generating ev charging scenarios via a denoising diffusion model.
\newblock \emph{IEEE Transactions on Smart Grid}, 2024.

\bibitem[Lim(2023)]{limgithub}
Jae~Hyun Lim.
\newblock Score-based diffusion models in function space.
\newblock \url{https://github.com/lim0606/ddo}, 2023.
\newblock Accessed: 2024-09-21.

\bibitem[Lim et~al.(2023)Lim, Kovachki, Baptista, Beckham, Azizzadenesheli, Kossaifi, Voleti, Song, Kreis, Kautz, et~al.]{lim2023score}
Jae~Hyun Lim, Nikola~B Kovachki, Ricardo Baptista, Christopher Beckham, Kamyar Azizzadenesheli, Jean Kossaifi, Vikram Voleti, Jiaming Song, Karsten Kreis, Jan Kautz, et~al.
\newblock Score-based diffusion models in function space.
\newblock \emph{arXiv preprint arXiv:2302.07400}, 2023.

\bibitem[Lyons(1998)]{lyons1998differential}
Terry~J Lyons.
\newblock Differential equations driven by rough signals.
\newblock \emph{Revista Matem{\'a}tica Iberoamericana}, 14\penalty0 (2):\penalty0 215--310, 1998.

\bibitem[Lyons \& Xu(2017)Lyons and Xu]{lyons2017hyperbolic}
Terry~J Lyons and Weijun Xu.
\newblock Hyperbolic development and inversion of signature.
\newblock \emph{Journal of Functional Analysis}, 272\penalty0 (7):\penalty0 2933--2955, 2017.

\bibitem[Lyons \& Xu(2018)Lyons and Xu]{lyons2018inverting}
Terry~J Lyons and Weijun Xu.
\newblock Inverting the signature of a path.
\newblock \emph{Journal of the European Mathematical Society}, 20\penalty0 (7):\penalty0 1655--1687, 2018.

\bibitem[Mandelbrot \& Van~Ness(1968)Mandelbrot and Van~Ness]{mandelbrot1968fractional}
Benoit~B Mandelbrot and John~W Van~Ness.
\newblock Fractional brownian motions, fractional noises and applications.
\newblock \emph{SIAM review}, 10\penalty0 (4):\penalty0 422--437, 1968.

\bibitem[Manten et~al.(2024)Manten, Casolo, Ferrucci, Mogensen, Salvi, and Kilbertus]{manten2024signature}
Georg Manten, Cecilia Casolo, Emilio Ferrucci, S{\o}ren~Wengel Mogensen, Cristopher Salvi, and Niki Kilbertus.
\newblock Signature kernel conditional independence tests in causal discovery for stochastic processes.
\newblock \emph{arXiv preprint arXiv:2402.18477}, 2024.

\bibitem[Mogren(2016)]{mogren2016c}
Olof Mogren.
\newblock C-rnn-gan: Continuous recurrent neural networks with adversarial training.
\newblock \emph{arXiv preprint arXiv:1611.09904}, 2016.

\bibitem[Morrill et~al.(2021)Morrill, Salvi, Kidger, and Foster]{morrill2021neural}
James Morrill, Cristopher Salvi, Patrick Kidger, and James Foster.
\newblock Neural rough differential equations for long time series.
\newblock In \emph{International Conference on Machine Learning}, pp.\  7829--7838. PMLR, 2021.

\bibitem[Pannier \& Salvi(2024)Pannier and Salvi]{pannier2024path}
Alexandre Pannier and Cristopher Salvi.
\newblock A path-dependent pde solver based on signature kernels.
\newblock \emph{arXiv preprint arXiv:2403.11738}, 2024.

\bibitem[Pei et~al.(2021)Pei, Ren, Yang, Liu, Qin, and Li]{pei2021towards}
Hengzhi Pei, Kan Ren, Yuqing Yang, Chang Liu, Tao Qin, and Dongsheng Li.
\newblock Towards generating real-world time series data.
\newblock In \emph{2021 IEEE International Conference on Data Mining (ICDM)}, pp.\  469--478. IEEE, 2021.

\bibitem[Phillips et~al.(2022)Phillips, Seror, Hutchinson, De~Bortoli, Doucet, and Mathieu]{phillips2022spectral}
Angus Phillips, Thomas Seror, Michael Hutchinson, Valentin De~Bortoli, Arnaud Doucet, and Emile Mathieu.
\newblock Spectral diffusion processes.
\newblock \emph{arXiv preprint arXiv:2209.14125}, 2022.

\bibitem[Popov et~al.(2021)Popov, Vovk, Gogoryan, Sadekova, and Kudinov]{popov2021grad}
Vadim Popov, Ivan Vovk, Vladimir Gogoryan, Tasnima Sadekova, and Mikhail Kudinov.
\newblock Grad-tts: A diffusion probabilistic model for text-to-speech.
\newblock In \emph{International Conference on Machine Learning}, pp.\  8599--8608. PMLR, 2021.

\bibitem[Ree(1958)]{ree1958lie}
Rimhak Ree.
\newblock Lie elements and an algebra associated with shuffles.
\newblock \emph{Annals of Mathematics}, 68\penalty0 (2):\penalty0 210--220, 1958.

\bibitem[Reizenstein(2017)]{reizenstein2017calculation}
Jeremy Reizenstein.
\newblock Calculation of iterated-integral signatures and log signatures.
\newblock \emph{arXiv preprint arXiv:1712.02757}, 2017.

\bibitem[Reutenauer(2003)]{reutenauer2003free}
Christophe Reutenauer.
\newblock Free lie algebras.
\newblock In \emph{Handbook of algebra}, volume~3, pp.\  887--903. Elsevier, 2003.

\bibitem[Rubanova et~al.(2019)Rubanova, Chen, and Duvenaud]{rubanova2019latent}
Yulia Rubanova, Ricky~TQ Chen, and David~K Duvenaud.
\newblock Latent ordinary differential equations for irregularly-sampled time series.
\newblock \emph{Advances in neural information processing systems}, 32, 2019.

\bibitem[Salvi et~al.(2021{\natexlab{a}})Salvi, Cass, Foster, Lyons, and Yang]{salvi2021signature}
Cristopher Salvi, Thomas Cass, James Foster, Terry Lyons, and Weixin Yang.
\newblock The signature kernel is the solution of a goursat pde.
\newblock \emph{SIAM Journal on Mathematics of Data Science}, 3\penalty0 (3):\penalty0 873--899, 2021{\natexlab{a}}.

\bibitem[Salvi et~al.(2021{\natexlab{b}})Salvi, Lemercier, Liu, Horvath, Damoulas, and Lyons]{salvi2021higher}
Cristopher Salvi, Maud Lemercier, Chong Liu, Blanka Horvath, Theodoros Damoulas, and Terry Lyons.
\newblock Higher order kernel mean embeddings to capture filtrations of stochastic processes.
\newblock \emph{Advances in Neural Information Processing Systems}, 34:\penalty0 16635--16647, 2021{\natexlab{b}}.

\bibitem[Salvi et~al.(2023)Salvi, Diehl, Lyons, Preiss, and Reizenstein]{salvi2023structure}
Cristopher Salvi, Joscha Diehl, Terry Lyons, Rosa Preiss, and Jeremy Reizenstein.
\newblock A structure theorem for streamed information.
\newblock \emph{Journal of Algebra}, 634:\penalty0 911--938, 2023.

\bibitem[Song \& Ermon(2019)Song and Ermon]{song2019generative}
Yang Song and Stefano Ermon.
\newblock Generative modeling by estimating gradients of the data distribution.
\newblock \emph{Advances in neural information processing systems}, 32, 2019.

\bibitem[Song et~al.(2020)Song, Sohl-Dickstein, Kingma, Kumar, Ermon, and Poole]{song2020score}
Yang Song, Jascha Sohl-Dickstein, Diederik~P Kingma, Abhishek Kumar, Stefano Ermon, and Ben Poole.
\newblock Score-based generative modeling through stochastic differential equations.
\newblock \emph{arXiv preprint arXiv:2011.13456}, 2020.

\bibitem[Stanley(2024)]{morganstanley2024msml}
Morgan Stanley.
\newblock Msml: Morgan stanley machine learning.
\newblock \url{https://github.com/morganstanley/MSML}, 2024.
\newblock Accessed: 2024-05-21.

\bibitem[Tashiro et~al.(2021)Tashiro, Song, Song, and Ermon]{tashiro2021csdi}
Yusuke Tashiro, Jiaming Song, Yang Song, and Stefano Ermon.
\newblock Csdi: Conditional score-based diffusion models for probabilistic time series imputation.
\newblock \emph{Advances in Neural Information Processing Systems}, 34:\penalty0 24804--24816, 2021.

\bibitem[Trottet et~al.(2023)Trottet, Sch{\"u}rch, Mollaysa, Allam, and Krauthammer]{trottet2023generative}
C{\'e}cile Trottet, Manuel Sch{\"u}rch, Amina Mollaysa, Ahmed Allam, and Michael Krauthammer.
\newblock Generative time series models with interpretable latent processes for complex disease trajectories.
\newblock In \emph{Deep Generative Models for Health Workshop NeurIPS 2023}, 2023.

\bibitem[{UCI Machine Learning Repository}(2024)]{uciml2024electric}
{UCI Machine Learning Repository}.
\newblock Electric power consumption dataset.
\newblock \url{https://www.kaggle.com/datasets/uciml/electric-power-consumption-data-set}, 2024.
\newblock Accessed: 2024-05-21.

\bibitem[Vaswani et~al.(2017)Vaswani, Shazeer, Parmar, Uszkoreit, Jones, Gomez, Kaiser, and Polosukhin]{vaswani2017attention}
Ashish Vaswani, Noam Shazeer, Niki Parmar, Jakob Uszkoreit, Llion Jones, Aidan~N Gomez, {\L}ukasz Kaiser, and Illia Polosukhin.
\newblock Attention is all you need.
\newblock \emph{Advances in neural information processing systems}, 30, 2017.

\bibitem[Vincent(2011)]{vincent2011connection}
Pascal Vincent.
\newblock A connection between score matching and denoising autoencoders.
\newblock \emph{Neural computation}, 23\penalty0 (7):\penalty0 1661--1674, 2011.

\bibitem[Voleti et~al.(2022)Voleti, Jolicoeur-Martineau, and Pal]{voleti2022mcvd}
Vikram Voleti, Alexia Jolicoeur-Martineau, and Chris Pal.
\newblock Mcvd-masked conditional video diffusion for prediction, generation, and interpolation.
\newblock \emph{Advances in neural information processing systems}, 35:\penalty0 23371--23385, 2022.

\bibitem[Yildiz et~al.(2019)Yildiz, Heinonen, and Lahdesmaki]{yildiz2019ode2vae}
Cagatay Yildiz, Markus Heinonen, and Harri Lahdesmaki.
\newblock Ode2vae: Deep generative second order odes with bayesian neural networks.
\newblock \emph{Advances in Neural Information Processing Systems}, 32, 2019.

\bibitem[Yoon(2024)]{yoon2024timegan}
Jinsung Yoon.
\newblock Timegan: Temporal generative adversarial networks for time-series synthesis - main script.
\newblock \url{https://github.com/jsyoon0823/TimeGAN/blob/master/main_timegan.py}, 2024.
\newblock Accessed: 2024-05-21.

\bibitem[Yoon et~al.(2019)Yoon, Jarrett, and Van~der Schaar]{yoon2019time}
Jinsung Yoon, Daniel Jarrett, and Mihaela Van~der Schaar.
\newblock Time-series generative adversarial networks.
\newblock \emph{Advances in neural information processing systems}, 32, 2019.

\bibitem[Yuan(2024)]{ydebugsys2024diffusionts}
Xinyu Yuan.
\newblock Diffusion-ts: Diffusion models for time series.
\newblock \url{https://github.com/Y-debug-sys/Diffusion-TS}, 2024.
\newblock Accessed: 2024-05-21.

\bibitem[Yuan \& Qiao(2024)Yuan and Qiao]{yuan2024diffusion}
Xinyu Yuan and Yan Qiao.
\newblock Diffusion-{TS}: Interpretable diffusion for general time series generation.
\newblock In \emph{The Twelfth International Conference on Learning Representations}, 2024.
\newblock URL \url{https://openreview.net/forum?id=4h1apFjO99}.

\bibitem[Zhou et~al.(2023)Zhou, Poli, Xu, Massaroli, and Ermon]{zhou2023deep}
Linqi Zhou, Michael Poli, Winnie Xu, Stefano Massaroli, and Stefano Ermon.
\newblock Deep latent state space models for time-series generation.
\newblock In \emph{International Conference on Machine Learning}, pp.\  42625--42643. PMLR, 2023.

\end{thebibliography}
\bibliographystyle{iclr2025_conference}

\newpage

\appendix

\section*{Appendix}

This appendix is structured in the following way. In \cref{sec:extra-sig} we complement the material presented in \cref{sec:sig-diff} with additional details on the signature. In \cref{appendix:orthogonal_polynomials}, we provide examples of orthogonal polynomial families one can use for signature inversion due to the derived inversion formulae in \cref{inversion}. In \cref{sec:proofs} we provide proofs for the signature inversion \cref{thm:fourier_inversion} and \cref{thm:poly_inversion}. \cref{appendix:visualisation} contains additional examples and discussion about the quality of signature inversion by different bases. \cref{appendix:experiments} provides details on the implementation of experiments.

\section{Additional Details on the Signature}\label{sec:extra-sig}
In this section, we establish the foundational algebraic framework for signatures in \cref{subsection:algebraic_setup}. We then provide a mathematically rigorous definition of the (log)signature in \cref{appendix:log-sig}, building upon the introduction in \cref{sec:sig-diff}. The section concludes with illustrative signature computation examples in \cref{sec:ex-comp}.

\subsection{Algebraic setup}\label{subsection:algebraic_setup}

For any positive integer $n \in \mathbb N$, we consider the \emph{truncated tensor algebra} over $\mathbb R^d$
\begin{equation*}
    T^n(\mathbb R^d) := \bigoplus_{k=0}^{n}(\mathbb R^d)^{\otimes k},
\end{equation*}
where $\otimes$ denotes the outer product of vector spaces. For any scalar $\alpha \in \mathbb R$, we denote by $T^n_\alpha(\mathbb R^d) = \{A \in T^n(\mathbb R^d) : A_0 = \alpha\}$ the hyperplane of elements in $T^n(\mathbb R^d)$ with the $0^{th}$ term equal to $\alpha$. 

$T^n(\mathbb R^d)$ is a non-commutative algebra when endowed with the tensor product $\cdot$ defined for any two elements $A=\left(A_{0}
,A_{1},...,A_n\right)  $ and $B=\left(  B_{0},B_{1},...B_n\right)$ 
of $T^n(\mathbb R^d)$ as follows
\begin{equation}
A \cdot B = \left(C_{0}, C_{1},...,C_n \right) \in T^n(\mathbb R^d),
 \quad \text{ where } \quad C_{k}=\sum_{i=0}^{k}A_{i} \otimes B_{k-i}\in (\mathbb R^d)^{\otimes k}.
 \label{eqn:tensor-product}
\end{equation}
The standard basis of $\mathbb R^d$ is denoted by $e_1,e_2,...,e_d$. We will refer to these basis elements as \emph{letters}. Elements of the induced standard basis of $T^n(\mathbb R^d)$ are often referred to as \emph{words} and abbreviated
$$e_{i_1i_2...i_k} = e_{i_1} \otimes e_{i_2} \otimes...\otimes e_{i_k}, \quad \text{for} \quad 1 \leq i_1,...,i_k \leq d \text{ and } 0 \leq k \leq n.$$
We will make use of the dual pairing notation $\langle e_{i_1i_2...i_k}, A \rangle \in \mathbb R$ to denote the $(i_1,...,i_k)^{th}$ element of a tensor $A \in T^n(\mathbb R^d)$. This pairing is extended by linearity to any linear combination of words. 

Following \citet{reutenauer2003free}, the truncated tensor algebra $T^n(\mathbb R^d)$ carries several additional algebraic structures. 

Firstly, it is a \emph{Lie algebra}, where the Lie bracket is the commutator
$$[A,B] = A \cdot B - B \cdot A \quad \text{for } A,B \in T^n(\mathbb R^d).$$
We denote by $\mathcal{L}^n(\mathbb R^d)$ the smallest Lie subalgebra of $T^n(\mathbb R^d)$ containing $\mathbb R^d$. We note that the Lie algebra $\mathcal{L}^n(\mathbb R^d)$ is a vector space of dimension $\beta(d, n)$ with
\begin{equation*}
    \beta(d, n) = \sum_{k = 1}^n \frac{1}{k} \sum_{i | k} \mu\left(\frac{k}{i}\right) d^i,
\end{equation*}
where $\mu$ is the M{\"o}bius function \citep{reutenauer2003free}. Bases of this space are known as \emph{Hall bases} \citep{reutenauer2003free, reizenstein2017calculation}. One of the most well-known bases is the \emph{Lyndon basis} indexed by \emph{Lyndon words}. A Lyndon word is a word occurring lexicographically earlier than any word obtained by cyclically rotating its elements.

Secondly, $T^n(\mathbb R^d)$ is also a commutative algebra with respect to the \emph{shuffle product} $\shuffle$. On basis elements, the shuffle product of two words of length $r$ and $s$ (with $r + s \leq n$) is the sum over the $\binom{r+s}{s}$ ways of interleaving the two words. For a more formal definition, see \citet[Section 1.4]{reutenauer2003free}. 

Related to the shuffle product is the \emph{right half-shuffle product} $\succ$ defined recursively as follows: for any two words $e_{i_1...i_r}$ and $e_{j_1...j_s}$ and letter $e_j$
\begin{equation*}
    e_{i_1...i_r}\succ e_{j} = e_{i_1...i_rj} \quad \text{and} \quad e_{i_1...i_r} \succ e_{j_1...j_s}=(e_{i_1...i_r} \succ e_{j_1...j_{s-1}} + e_{j_1...j_{s-1}} \succ e_{i_1...i_r})\cdot e_{j_s}.
\end{equation*}
The right half-shuffle product will be useful for carrying out computations in the next section. Note that the following relation between shuffle and right half-shuffle products holds \citep{salvi2023structure}
\begin{equation*}
    e_{i_1...i_r} \shuffle e_{j_1...j_s} = e_{i_1...i_r} \succ e_{j_1...j_s} + e_{j_1...j_s} \succ e_{i_1...i_r}.
\end{equation*}
Equipped with this algebraic setup, we can now introduce the signature.

\subsection{The (log)signature}
\label{appendix:log-sig}
Let $x : [0,1] \to \mathbb{R}^d$ be a smooth path. The \emph{step-$n$ signature} $S^{\leq n}(x)$ of $x$ is defined as the following collection of iterated integrals
\begin{equation}\label{eqn:step-n-sig}
S^{\leq n}(x) = \left(1, S_1(x)  , \ldots, S_n(x)\right) \in T_1^n(\mathbb R^d)
\end{equation}
where
\begin{equation*}
    S_k(x)=\int_{0\leq t_{1}<...<t_{k}\leq 1} dx_{t_{1}} \otimes ... \otimes dx_{t_{k}} \in (\mathbb R^d)^{\otimes k} \quad \text{for } 1 \leq k \leq n.
\end{equation*}

An important property of the signature is usually referred to as the \emph{shuffle identity}. This result is originally due to \citet{ree1958lie}. For a modern proof see \cite[Theorem 1.3.10]{cass2024lecture}.
\begin{lemma}[Shuffle identity] \citep{ree1958lie}\label{thm:shuffle}
    Let $x : [0,1] \to \mathbb{R}^d$ be a smooth path. For any two words $e_{i_1...i_r}$ and $e_{j_1...j_s}$, with $0 \leq r, s \leq n$, the following two identities hold
    \begin{equation*}
        \left\langle e_{i_1...i_r} \shuffle e_{j_1...j_s}, S^{\leq n}({x})\right\rangle = \left\langle e_{i_1...i_r}, S^{\leq n}({x})\right\rangle \left\langle e_{j_1...j_s}, S^{\leq n}({x})\right\rangle,
    \end{equation*}
    \begin{equation*}
       \left\langle e_{i_1...i_r} \succ e_{j_1...j_s}, S^{\leq n}({x})\right\rangle = \int_0^1 \left\langle e_{i_1...i_r}, S^{\leq n}({x})_t \right\rangle d \left\langle e_{j_1...j_s}, S^{\leq n}({x})_t\right\rangle,
    \end{equation*}
    where $S^{\leq n}(x)_t$ is the step-$n$ signature of the path $x$ restricted to the interval $[0,t]$.
\end{lemma}
An example of simple computations using the shuffle identity is presented in \cref{sec:ex-comp}.

Moreover, it turns out that the signature is more than just a generic element of $T_1^n(\mathbb R^d)$; in fact, its range has the structure of a Lie group as we shall explain next. Recall that the tensor exponential $\exp$ and the tensor logarithm $\log$ are maps from $T^n(\mathbb R^d)$ to itself defined as follows
\begin{equation}
    \exp (A) := \sum_{k\geq 0} \frac{1}{k!}A^{\otimes k} \quad \text{and} \quad \log(\mathbf{1} + A) = \sum_{k\geq 1}\frac{(-1)^{k-1}}{k}(A)^{\otimes k} 
    \label{eqn:log}
\end{equation}
where $\mathbf{1} = (1,0,...,0) \in T^n(\mathbb R^d).$
It is a well-known fact that $\exp : T^n_0(\mathbb R^d) \to T^n_1(\mathbb R^d)$ and $\log : T^n_1(\mathbb R^d) \to T^n_0(\mathbb R^d)$ are mutually inverse.

The \emph{step-$n$ free nilpotent Lie group} is the image of the free Lie algebra under
the exponential map
\begin{equation}
    \mathcal{G}^n(\mathbb R^d) = \exp(\mathcal{L}^n(\mathbb R^d)) \subset T^n_1(\mathbb R^d).
\end{equation}
As its name suggests, $\mathcal{G}^n(\mathbb R^d)$ is a Lie group and plays a central role in the theory of rough paths \citep{friz2010multidimensional}.

Here comes the connection with signatures. It is established by the following fundamental result due to \citet{chen1957integration, chen1958integration}, which can also be viewed as a consequence of Chow’s results in \citet{chow1939system}.

\begin{lemma}[Chen–Chow] \citep{chen1957integration, chen1958integration, chow1939system}\label{thm:chen-chow}
    The step-$n$ free nilpotent Lie group $\mathcal{G}^n(\mathbb R^d)$ is precisely the image of the step-$n$ signature map in Equation (\ref{eqn:step-n-sig}) when the latter is applied to all smooth paths in $\mathbb R^d$
    \begin{equation*}
        \mathcal{G}^n(\mathbb R^d) = \{S^{\leq n}(x) \mid x : [0,1] \to \mathbb R^d \ \text{smooth}\}.
    \end{equation*}
\end{lemma}

\subsection{Simple examples of signature computations}\label{sec:ex-comp}
In the following examples, we alter the notation so that for a path $x:[a,t]\rightarrow\mathbb{R}^d$, the tensor representing the $k$-th level of the signature computed on an interval $[a,t]$ is denoted as
\begin{equation}\label{niters}
    S(x)^{(k)}_{a, t} = (S(x)^{i_1, \dots, i_k}_{a, t}: i_1, \dots, i_k\in\{1, \dots, d\}) \in (\mathbb R^d)^{\otimes k}.
\end{equation}
Furthermore, we can express the value of $S(x)^{(k)}_{a, t}$ at a particular set of indices $i_1, \dots, i_k\in \{1, \dots d\}$ as a \textit{k-fold iterated integral}
\begin{equation}
    S(x)^{i_1, \dots, i_k}_{a, t}=\int_{a<t_1<\dots<t_k<t}dx^{i_1}_{t_1}\dots dx^{i_k}_{t_k}.
\end{equation}
We assume that the signature is always truncated at a sufficiently high level $n$, allowing us to denote the step-$n$ signature simply as

\begin{equation}\label{def_sig}
    S(x)_{a, t} = (1, S(x)^{(1)}_{a, t}, S(x)^{(2)}_{a, t}, S(x)^{(3)}_{a, t}, \dots, S(x)^{(n)}_{a, t})\in T_1^n(\mathbb R^d).
\end{equation}

\begin{example}[Geometric interpretation of a 2-dimensional path]\label{signed_area_example} Consider a path $\hat x:[0,9]\rightarrow \mathbb R^2$, where $\hat x=(x_t^1,x_t^2)=(t,x(t))$. Here, $x(t)$ is defined as
\begin{equation*}
   x_t^2 =x(t) = \begin{cases}
    \sqrt{3}t \qquad &t\in[0, 2]\\
    2\sqrt{3} \qquad &t\in[2, 8]\\
    \sqrt{3}t-6\sqrt{3}\qquad &t\in[8, 9]
    \end{cases},
\end{equation*}
which is continuous and piecewise differentiable. In this case, $\dot x_t^1=1$, and $\dot x_t^2$ can be expressed as
\begin{equation*}
    \dot x_t^2 = \dot x(t) = \begin{cases}
    \sqrt{3} \qquad &t\in(0, 2)\\
    0 \qquad &t\in(2, 8)\\
    \sqrt{3}\qquad &t\in(8, 9)
    \end{cases}.
\end{equation*}
One can compute the step-$n$ signature of $\hat x$ as
\begin{align*}
    S(\hat x)_{0, 9} &= (1, S(\hat x)^{(1)}_{0, 9}, S(\hat x)^{(2)}_{0, 9}, S(\hat x)^{(3)}_{0, 9}, \dots, S(\hat x)^{(n)}_{0, 9})\\
    &=(1, S(\hat x)^{1}_{0, 9}, S(\hat x)^{2}_{0, 9}, S(\hat x)^{1, 2}_{0, 9}, S(\hat x)^{2,1}_{0, 9}, S(\hat x)^{1, 1, 1}_{0, 9}, \dots, S(\hat x)_{0,9}^{i_1, \ldots, i_n}),
\end{align*}
where
\begin{align*}
    &S(\hat x)^{1}_{0, 9} = \int_{0<s<9}dx^1_s=x^1_{9}-x^1_{0} = 9\\
    &S(\hat x)^{2}_{0, 9} = \int_{0<s<9}dx^2_s=x^2_{9}-x^2_{0} = 3\sqrt{3}\\
    &S(\hat x)^{1, 1}_{0, 9} = \int_{0<r<s<9}dx^1_rdx^1_s=\int_{0<s<9}x^1_sdx^1_s=\frac{1}{2}\left(x^1_s\right)^2\biggr|_{0}^{9}=\frac{81}{2}\\
    &S(\hat x)^{1, 2}_{0, 9} = \int_{0<r<s<9}dx^1_rdx^2_s=\int_{0<s<9}sdx^2_s=\int_{0<s<9}s\dot x^2_s ds=\frac{\sqrt{3}}{2}s^2\biggr|_{0}^{2}+\frac{\sqrt{3}}{2}s^2\biggr|_{8}^{9}=\frac{21}{2}\sqrt{3}\\
    &S(\hat x)^{2, 1}_{0, 9} = \int_{0<r<s<9}dx^2_rdx^1_s=\int_{0<s<9}x^2_sds=\frac{\sqrt{3}}{2}s^2\biggr|_{0}^{2}+2\sqrt{3}s\biggr|_{2}^{8}+\frac{\sqrt{3}}{2}s^2-6\sqrt{3}s\biggr|_{8}^{9}=\frac{33}{2}\sqrt{3}\\
    &S(\hat x)^{2, 2}_{0, 9} = \int_{0<r<s<9}dx^2_rdx^2_s = \int_{0<s<9}x^2_sdx^2_s = \frac{1}{2}\left(x^2_s\right)^2\biggr|_{0}^{9} = \frac{27}{2}.
\end{align*}
From Figure \ref{fig:signed_area}, let $A_-$ and $A_+$ represent the signed value of the shaded region. The signed Lévy area of the path is defined as $A_-+A_+$. In this case, the signed Lévy area is $-3\sqrt{3}$. Surprisingly, 
\begin{equation*}
    \frac{1}{2}\left(S(\hat x)^{1, 2}_{0, 9}-S(\hat x)^{2, 1}_{0, 9}\right) = \frac{1}{2}\left(\frac{21}{2}\sqrt{3}-\frac{33}{2}\sqrt{3}\right)=-3\sqrt{3}=A_-+A_+,
\end{equation*}
which is exactly the signed Lévy area.
\begin{figure}[hbtp]
    \centering
    \setlength{\abovecaptionskip}{0pt}
    \includegraphics[width=0.8\linewidth]{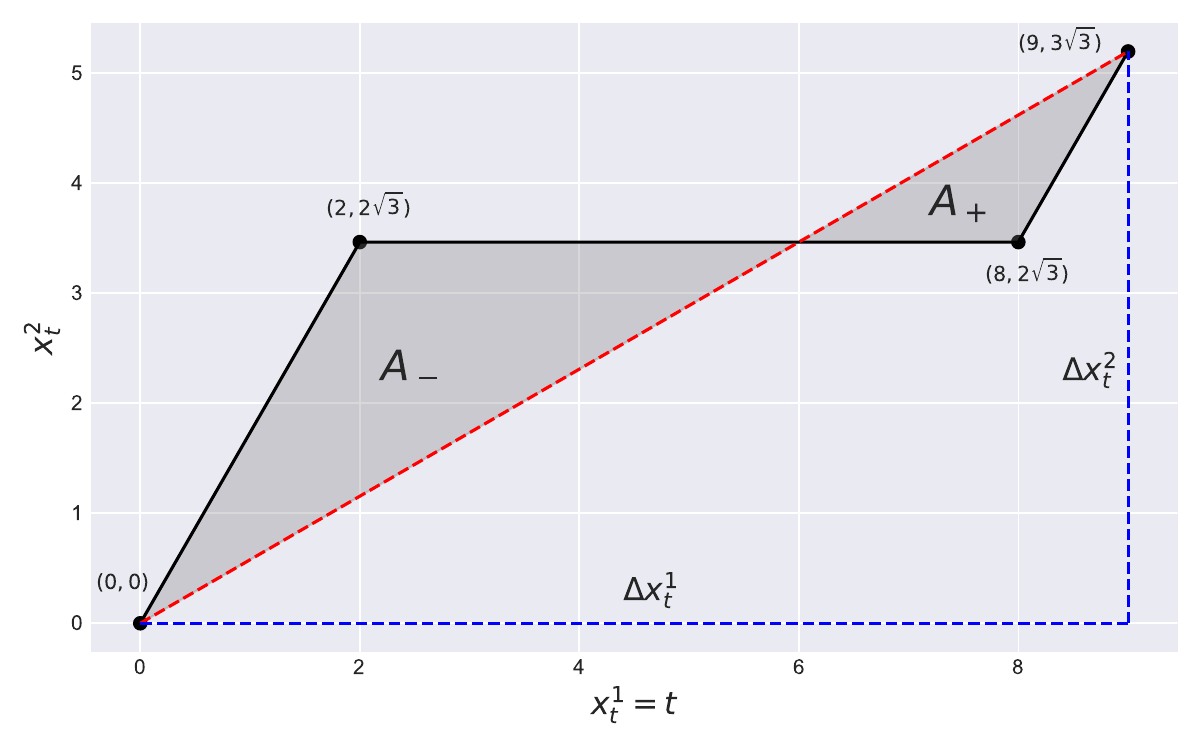}
    \caption{Path in Example \ref{signed_area_example}. The shaded region represents the signed Lévy area.}
    \label{fig:signed_area}
\end{figure}
\end{example}

Another important example is given by the signature of linear paths.

\begin{example}[Signatures of linear paths]\label{example_linear_sig}
    Suppose there is a linear path $x : [a,b] \to \mathbb R^d$. Then the path $x$ is linear in terms of $t$, i.e. 
    \begin{equation*}
        x_t=x_a + \frac{t-a}{b-a}\left(x_b-x_a\right).
    \end{equation*}
    It follows that its derivative can be written as
    \begin{equation*}
        dx_t=\frac{\left(x_b-x_a\right)}{b-a}dt.
    \end{equation*}
    Recalling the definition of a signature, it holds that
    \begin{align*}
        S(x)^{i_1, \dots, i_k}_{a, b}&=\int_{a<t_1<\dots<t_k<b}dx^{i_1}_{t_1}\dots dx^{i_k}_{t_k}\\
        &=\frac{\prod_{j=1}^k\left(x_b^{i_j}-x_a^{i_j}\right)}{(b-a)^k}\int_{a<t_1<\dots<t_k<b}dt_1\dots dt_k\\
        &=\frac{\prod_{j=1}^k\left(x_b^{i_j}-x_a^{i_j}\right)}{(b-a)^k}\frac{(b-a)^k}{k!}\\
        &=\frac{\prod_{j=1}^k\left(x_b^{i_j}-x_a^{i_j}\right)}{k!}.
    \end{align*}
    Therefore, the whole step-$n$ signature can be expressed as a tensor exponential of the linear increment $x_b-x_a$
    \begin{align*}
        S(x)^{(k)}_{a, b} &= \frac{\left(x_b-x_a\right)^{\otimes k}}{k!},\\
        S(x)_{a, b} &=\sum_{k=0}^n \frac{\left(x_b-x_a\right)^{\otimes k}}{k!}\\
        &=\exp_\otimes \left(x_b-x_a\right).
    \end{align*}
\end{example}

Chen's identity in \cref{thm:chen} is one of the most fundamental algebraic properties of the signature as it describes the behaviour of the signature under the concatenation of paths.
\begin{definition}[Concatenation]
    Consider two smooth paths $x : [a,b] \to \mathbb R^d$ and $y : [b,c]\to \mathbb R^d$. Define the \textit{concatenation} of $x$ and $ y$, denoted by $x\ast y$ as a path $[a, c]\rightarrow  \mathbb R^d$
    \begin{equation*}
        (x\ast y)_t :=\begin{cases}
            x_t\qquad &\text{ if } a\leq t\leq b\\
            x_b-{y}_b+ y_t\qquad &\text{ if } b\leq t\leq c
        \end{cases}.
    \end{equation*}
\end{definition}
Chen's identity in \cref{thm:chen} provides a method to simplify the analysis of longer paths by converting them into manageable shorter ones. If we have a smooth path $x:[t_0, t_n]\rightarrow \mathbb R^d$, then inductively, we can decompose the signature of $x$ to
\begin{equation*}
    S(x)_{t_0, t_n}=S(x)_{t_0, t_1}\cdot S(x)_{t_1, t_2}\cdot \cdots \cdot S(x)_{t_{n-1}, t_n}.
\end{equation*}
Moreover, if we have a time series $(t_0,x_0), ..., (t_n, x_n) \in \mathbb R^{d+1}$, we can treat $x$ as a piecewise linear path interpolating the data. Based on Example \ref{example_linear_sig}, one can observe that
\begin{equation*}
    S(x)_{t_0, t_n}=\exp_\otimes \left(x_{t_1}-x_{t_0}\right)\cdot\exp_\otimes \left(x_{t_2}-x_{t_1}\right)\cdot \cdots \cdot \exp_\otimes \left(x_{t_n}-x_{t_{n-1}}\right),
\end{equation*}
which is widely used in Python packages such as \texttt{esig} or \texttt{iisignature}. 

\begin{example}[Example of shuffle identity]
    Consider a smooth path $x=(x^1_t,x^2_t) : [a,b] \to \mathbb R^2$. By integration by parts, we have
    \begin{align*}
        \langle e_1, S(x)_{a,b} \rangle \langle e_2, S(x)_{a,b} \rangle  &= \int_{a<t<b}dx^1_t\int_{a<t<b}dx^2_t\\
        &=\int_{a<t<b}\dot x^1_tdt\int_{a<t<b}\dot x^2_tdt\\
        &\overset{\mathrm{by\, parts}}=\int_{a<t<b}\langle e_2, S(x)_{a,t} \rangle \dot x^1_tdt+\int_{a<t<b}\langle e_1, S(x)_{a,t} \rangle \dot x^2_tdt\\
        &= \langle e_{2,1}, S(x)_{a,b} \rangle + \langle e_{1,2}, S(x)_{a,b} \rangle .
    \end{align*}
    By the shuffle identity, we have
    \begin{align*}
        \langle e_1, S(x)_{a,b} \rangle \langle e_2, S(x)_{a,b} \rangle &=\langle e_1, S({x})_{a, b}\rangle \langle e_2, S({x})_{a, b}\rangle\\
        &=\langle e_1\shuffle e_2, S({x})_{a, b}\rangle\\
        &=\langle e_{1,2}+e_{2,1}, S({x})_{a, b}\rangle\\
        &=\langle e_{1,2}, S(x)_{a,b}\rangle  +\langle e_{2,1}, S(x)_{a,b}\rangle,
    \end{align*}
    which is exactly the same as what we derived via integration by parts.
\end{example}

\begin{example}[Example of half-shuffle computations]
Consider a two-dimensional, real-valued smooth path $\hat x=(t, x(t)):[a, b]\rightarrow\mathbb R^2$ with $x(a)=0$. The elements of the first truncation level of the signature can be retrieved as follows:
\begin{align*}
    \langle e_1, S(\hat x)_{a,t} \rangle = \int_a^tds=t-a, \quad
    \langle e_2, S(\hat x)_{a,t} \rangle = \int_a^td\left(x(s)\right)=x(t)-x(a)=x(t).
\end{align*}
Then, one can express all integrals in terms of powers of $t-a$ and $x(t)$ by signatures of $\hat x$. For example, let $n, m\in \mathbb N_0$,
\begin{align*}
    \int^b_a (t-a)^nx(t)^m dt &= \int^b_a (t-a)^nx(t)^m d(t-a)\\
    &=\int^b_a \left(\langle e_1, S(\hat x)_{a,t} \rangle\right)^n\left(\langle e_2, S(\hat x)_{a,t} \rangle\right)^m d\left(\langle e_1, S(\hat x)_{a,t} \rangle\right)\\
    &= \langle\left(e_1^{\shuffle n}\shuffle e_2^{\shuffle m}\right)\succ e_1, S(\hat x)_{a,b}\rangle.
\end{align*}
\end{example}

\section{Orthogonal Polynomials and Fourier Series}\label{appendix:orthogonal_polynomials}

In this section, we introduce the background material on orthogonal polynomials and the Fourier series necessary for the signature inversion formulae presented in the next section.

\subsection{Orthogonal polynomials}

\subsubsection{Inner product and orthogonality}

Consider a dot product $(  x,  y)=\sum_{i=1}^nx_iy_i$, where ${x}, {y}\in \mathbb R^n$. For some weights $w_1, \cdots, w_n\in \mathbb R_+$, we can also define the weighted dot product $( x,  y)_w=\sum_{i=1}^nw_ix_iy_i$, where $(\cdot, \cdot)_w$ can be written as $(\cdot, \cdot)$ for simplicity. 

For $p\in[1, \infty)$, $L^p_w(\Omega)$ is the linear space of measurable functions from $\Omega$ to $\mathbb R$ such that their weighted $p$-norms are bounded, i.e.
\begin{equation*}
    L^2_w(\Omega) = \left\{v \text{ is measurable in }\Omega\biggr|\int_{\Omega}|v(t)|^2w(t)dt<\infty\right\}.
\end{equation*}
For example, let $d\alpha$ be a non-negative Borel measure supported on the interval $[a, b]$ and $\mathbb V=L^2_w(a, b)$. One can define $(f, g) = \int_a^bf(t)g(t)d\alpha(t)$ as a Stieltjes integral for all $f, g\in \mathbb V$. Note that if $\alpha(t)$ is absolutely continuous, which will be the setting throughout this section, then one can find a weight density $w(t)$ such that $d\alpha(t)=w(t)dt$. In this case, the definition of the inner product over a function space reduces to an integral with respect to a weight function, i.e.,
\begin{equation*}
    ( f, g) = \int_a^bf(t)g(t)w(t)dt.
\end{equation*}
We can then consider an orthogonal polynomial system to be orthogonal with respect to the \textit{weight} function $w$. We denote $\mathbb P[t]\subset L^2_w(\Omega)$ as the space of all polynomials. A polynomial of degree $n$, $p\in\mathbb P_n[t]$, is \textit{monic} if the coefficient of the $n$-th degree is one. 
\begin{definition}[Orthogonal polynomials]\label{orthpoly}
For an arbitrary vector space $\mathbb V$, $u$ and $v$ are \textit{orthogonal} if $(u, v)=0$ for all $u, v\in \mathbb V$. When $\mathbb V =\mathbb P[t]$, a sequence of polynomials $(p_n)_{n\in\mathbb N}\in\mathbb P[t]$ is called orthogonal with respect to a weight function $w$ if for all $m\neq n$,
\begin{equation*}
    (p_n, p_m)= \int p_n(t)p_m(t)w(t)dt = 0,
\end{equation*}
where $\text{deg}(p_n)=n$ is the degree of a polynomial. Furthermore, we say the sequence of orthogonal polynomials is \textit{orthonormal} if $( p_n, p_n) = 1$ for all $n\in\mathbb N$.
\end{definition}
For simplification, the inner product notation $(\cdot,\cdot)$ will be used without specifying the integral formulation for the orthogonal polynomials. To construct a sequence of orthogonal polynomials in Definition \ref{orthpoly}, one can follow the Gram-Schmidt orthogonalisation process, which is stated below.
\begin{theorem}[Gram-Schmidt orthogonalisation]\label{orthoGS_thm}
The polynomial system $(p_n)_{n\in\mathbb N}$ with respect to the inner product $( \cdot, \cdot)$ can be constructed recursively by
\begin{equation}\label{orthoGS_eq}
    p_0=1,\qquad p_n=t^n-\sum_{i=1}^{n-1}\frac{( t^n, p_i)}{( p_i, p_i)}p_i \quad\text{for } n\geq 1.
\end{equation}
\end{theorem}

From the orthogonalisation process in \cref{orthoGS_thm}, we can see that the $n$-th polynomial $p_n$ has degree $n$ exactly, which means $(p_n)_{n\in\mathbb N}$ is a basis spanning $\mathbb P[t]$. Furthermore, the orthogonal construction makes the orthogonal polynomial system an orthogonal basis with respect to the corresponding inner product. The following proposition forms an explicit expression for coefficients of $(p_k)_{k\in \{0, \cdots, n\}}$ in an arbitrary $n$-th degree polynomial.
\begin{proposition}[Orthogonal polynomial expansion]\label{ortho_expansion}
Consider an arbitrary polynomial $x(t)\in \mathbb P_n[t]$. One can express $x(t)$ by a sequence of orthogonal polynomials $(p_k)_{k\in \{0, \cdots, n\}}$, i.e.,
\begin{equation*}
    x(t) = \sum_{k=0}^{n} \frac{( p_k, x)}{( p_k, p_k)}p_k(t).
\end{equation*}
\end{proposition}

\begin{remark}
We have stated the orthogonal polynomial expansion for $x\in \mathbb P_n[t]$. In general, by the closure of orthogonal polynomial systems in $L^2_w(a, b)$, arbitrary $f\in L^2_w(a, b)$ can be written as an infinite sequence of orthogonal polynomials.
\begin{equation*}
    f(t) = \sum_{k=0}^{\infty} \frac{( p_k, f)}{( p_k, p_k)}p_k(t).
\end{equation*}
The $N$-th degree approximation of $f$ is the best approximating polynomial with a degree less or equal to $N$, denoted by
\begin{equation}\label{ortho_projection}
    P_Nf(t)=\sum_{k=0}^{N} \frac{( p_k, f)}{( p_k, p_k)}p_k(t).
\end{equation}
\end{remark}

\subsubsection{Basic properties}
Here, we will list the main properties of orthogonal polynomials significant for our application.
\subsubsection*{The three-term recurrence relation}
\begin{theorem}[Three-term recurrence relation]\label{3recurrence}
    A system of orthogonal polynomials $(p_n)_{n\in\mathbb N}$ with respect to a weight function $w$ satisfies the three-term recurrence relation.
    \begin{align*}
    p_0(t) = 1, \quad p_1(t) = A_1 t +B_1, \quad p_{n+1}(t) = (A_{n+1}t+B_{n+1})p_{n}(t) + C_{n+1}p_{n-1}(t),
    \end{align*}
    for all $n\in\mathbb N$, and $A_i>0$ for all $i\in\mathbb N_0$.
\end{theorem}
Before proving the recurrence relation, we will first show that an orthogonal polynomial is orthogonal to all polynomials with a degree lower than that of itself.
\begin{lemma}\label{0lower_poly}
A polynomial $q(t)\in\mathbb P_n[t]$ satisfies $( q, r)=0$ for all $r(t)\in\mathbb P_m[t]$ with $m<n$ if and only if $q(t)=p_n(t)$ up to some constant coefficient, where $p_n(t)$ denotes the orthogonal polynomial with degree $n$.
\end{lemma}
\begin{proof}
$\Longrightarrow$: Consider $q(t)=\alpha_nt^n+O(t^{n-1})$ and $p_n(t)=\tilde \alpha_nt^n+O(t^{n-1})$. Then we define
\begin{equation*}
    s(t) = q(t) - \frac{ \alpha_n}{\tilde\alpha_n} p_n(t) = O(t^{n-1}),
\end{equation*}
which has a degree at most $n-1$. Therefore, for all $m<n$,
\begin{equation*}
    ( s, p_m) = ( q, p_m) - \frac{ \alpha_n}{\tilde\alpha_n}( p_n, p_m)=0.
\end{equation*}
The former inner product $( q, p_m)=0$ by assumption, while the latter inner product $( p_n, p_m)=0$ by orthogonality. By Proposition \ref{ortho_expansion},
\begin{equation*}
    s(t) = \sum_{m=0}^{n-1} \frac{( p_m, s)}{( p_m, p_m)}p_m(t) = 0 \quad\implies\quad q(t) = \frac{\tilde \alpha_n}{\alpha_n} p_n(t).
\end{equation*}
$\Longleftarrow$: Consider $r(t)=\sum_{k=0}^{m}r_kp_k(t)$. Let $q(t)=cp_n(t)$. Using the linearity of the inner product and orthogonality of $(p_n)_{n\in\mathbb N}$, for all $m<n$,
\begin{equation*}
    ( q, r) = \left( cp_n(t), \sum_{k=0}^{m}r_kp_k(t)\right) = c\sum_{k=0}^{m}r_k( p_n(t),p_k(t)) = 0.
\end{equation*}
\end{proof}
Now, we have enough tools to prove the famous three-term recurrence relation.
\begin{proof}[Proof of \cref{3recurrence}]
Consider a sequence of orthogonal polynomials $(p_n)_{n\in\mathbb N}$. When $n=1$, $p_1$ can be expressed as $A_1t+B_1$ for $A_1, B_1\in\mathbb R$. This is because $p_1$ is an element in an orthogonal basis with degree $1$. Based on the inner product of orthogonal polynomials,
\begin{equation*}
    ( p_k, tp_n) =  \int tp_k(t)p_n(t)w(t)dt = ( tp_k, p_n).
\end{equation*}
Therefore, for $0\leq k<n-1$, we have $( p_k, tp_n)=0$ by \cref{0lower_poly}. Since $tp_n(t)$ has degree $n+1$, by Proposition \ref{ortho_expansion},
\begin{align*}
    &tp_n(t) = \sum_{k=0}^{n+1} \frac{( p_k, tp_n)}{( p_k, p_k)}p_k(t) = \sum_{k=n-1}^{n+1} \frac{( p_k, tp_n)}{( p_k, p_k)}p_k(t) = \alpha_{n-1}p_{n-1}(t)+\alpha_{n}p_{n}(t) + \alpha_{n+1}p_{n+1}(t)\\
    &\implies p_{n+1} = \left(\frac{1}{\alpha_{n+1}}t-\frac{\alpha_n}{\alpha_{n+1}}\right)p_n(t)-\frac{\alpha_{n-1}}{\alpha_{n+1}}p_{n-1}(t).
\end{align*}
\end{proof}
\begin{remark}
Recurrence is the core property of orthogonal polynomials in our setting, as one can find higher-order coefficients based on lower-order coefficients given the analytical form of the orthogonal polynomials. This idea coincides with the shuffle identity of signatures. As stated in \cref{thm:poly_inversion}, one can construct an explicit recurrence relation for the coefficients of orthogonal polynomials by linear functionals acting on signatures.
\end{remark}

\subsubsection*{Approximation results for functions in $L_w^2$}
Without loss of generality, consider $f\in L^2_w(-1, 1)$, as we can always transform an arbitrary interval $[a, b]$ linearly into the interval $[-1, 1]$. Recall the $N$-th degree approximation $P_Nf(t)$ defined in Equation (\ref{ortho_projection}). The uniform convergence of the $N$-th degree approximation $P_Nf(t)$ to $f$ can be found in \citet{the_num_ana}, where we obtain
\begin{equation*}
    \frac{1}{\sqrt{2\pi}}\|f-P_Nf\|_{2}\leq\|f-P_Nf\|_{\infty} \leq (1+ \|P_N\|)\|f-q\|_{\infty}, \qquad q\in\mathbb P_N,
\end{equation*}
where $\|P_N\|$ relates to the system of orthogonal polynomials, and $\|f-q\|_{\infty}$ depends on the smoothness of $f$. In the case of Chebyshev polynomials, where the weight function is $w(t)=1/\sqrt{1-t^2}$, $\|P_N\|=\frac{4}{\pi}\log n+\mathcal{O}(1)$ \citep{the_num_ana}. For some $\alpha\in(0, 1]$,
\begin{equation*}
    \|f-P_Nf\|_{2} \leq c_k\frac{\log N}{N^{k+\alpha}}\qquad \text{ for } N\geq 2.
\end{equation*}
This bound result is shown numerically in Figure \ref{fig:ortho_converge}.

\begin{figure}[htbp]
    \centering
    \includegraphics[width=\linewidth]{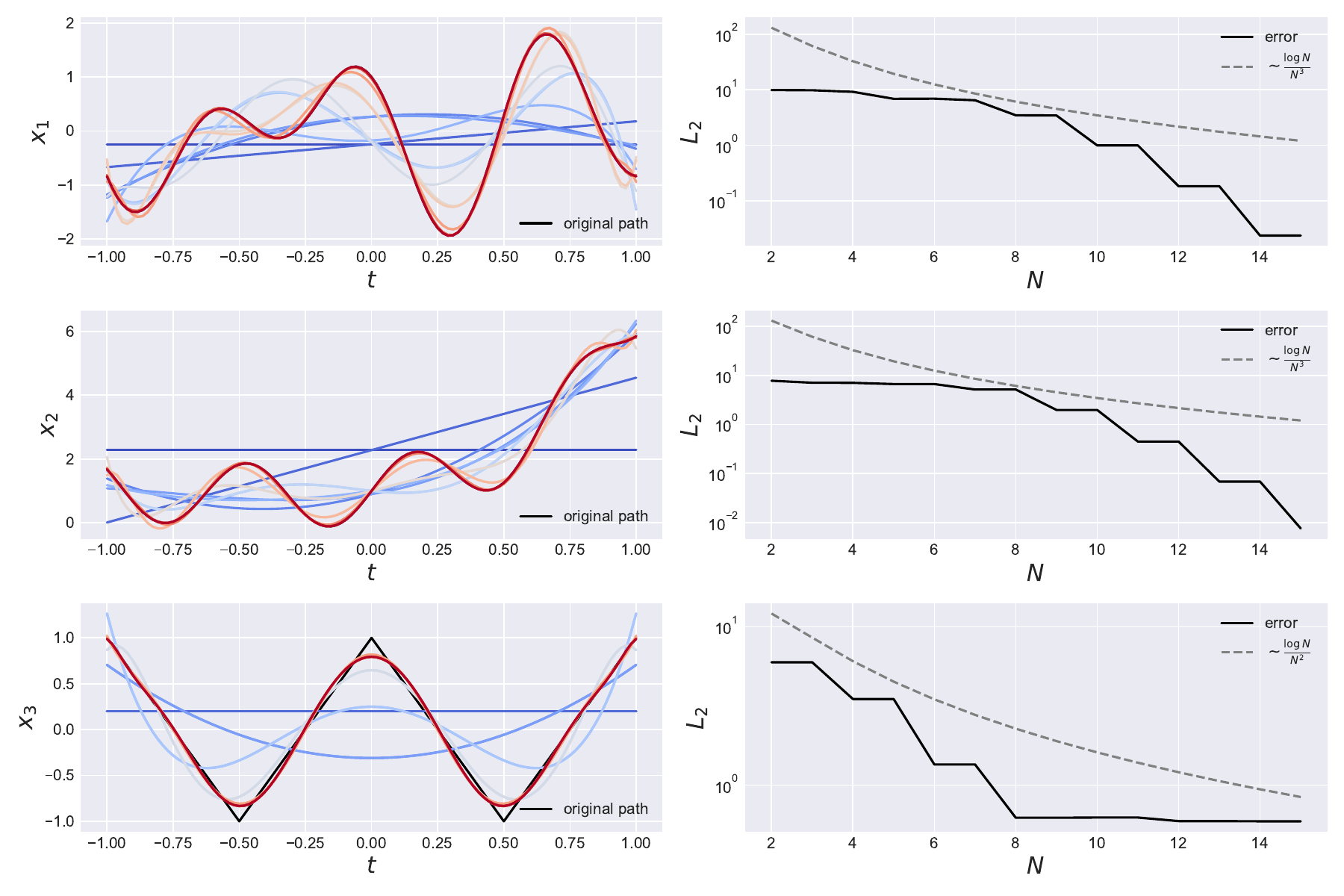}
    \caption{The approximation quality (left) and convergence of the $L_2$ error (right) for Chebyshev polynomials of increasing degree $N$. As $N$ increases, the colours change from blue to red in the left column. Paths are given by (top to bottom): $x_1(t)=\cos(10t)-\sin(2\pi t)$, $x_2(t)=\sin(10t)+e^{2t}-t$, $x_3(t)=2|2t-1|-1$.}
    \label{fig:ortho_converge}
\end{figure}

\subsubsection{Examples}\label{orth_poly_examples_j_h}
In this subsection, we will provide two general orthogonal polynomial families, Jacobi polynomials and Hermite polynomials, which will be used for signature inversion in the next section. Figure \ref{fig:ortho_examples} visualises the first few polynomials of these two kinds.
\begin{figure}[htbp]
    \centering
    \setlength{\abovecaptionskip}{0pt}
    \includegraphics[width=1\linewidth]{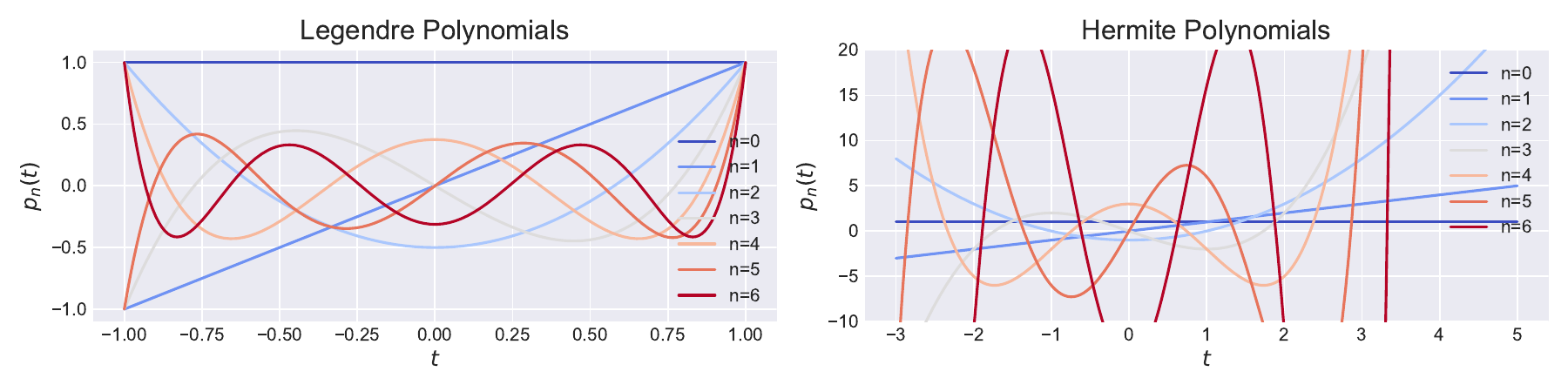}
    \caption{Visualisation of the first $7$ Legendre and Hermite polynomials.}
    \label{fig:ortho_examples}
\end{figure}
\subsubsection*{Jacobi polynomials}
Jacobi polynomials $p_n^{(\alpha, \beta)}$ are a system of orthogonal polynomials with respect to the weight function $w: (-1, 1)\rightarrow \mathbb R$ such that
\begin{equation*}
    w(t; \alpha, \beta)= (1-t)^\alpha(1+t)^\beta.
\end{equation*}
There are many well-known special cases of Jacobi polynomials, such as Legendre polynomials $p_n^{(0, 0)}$ and Chebyshev polynomials $p_n^{(-1/2, -1/2)}$. In general, the analytical expression of Jacobi polynomials \citep{quantum_ortho} is defined by the hypergeometric function ${}_2F_1$:
\begin{equation*}
    p_n^{(\alpha, \beta)}(t) = \frac{(\alpha+1)_n}{n!}{}_2F_1(-n, 1+\alpha+\beta+n; \alpha+1;\frac{1}{2}(1-t)),
\end{equation*}
where $(\alpha+1)_n$ is the Pochhammer's symbol. For orthogonality, Jacobi polynomials satisfy
\begin{equation*}
    \int_{-1}^1(1-t)^\alpha(1+t)^\beta p_m^{(\alpha, \beta)}(t)p_n^{(\alpha, \beta)}(t)=\frac{2^{\alpha+\beta+1}\Gamma(n+\alpha+1)\Gamma(n+\beta+1)}{(2n+\alpha+\beta+1)\Gamma(n+\alpha+\beta+1)n!}\delta_{nm}, \qquad \alpha, \beta>-1,
\end{equation*}
where $\delta_{mn}$ is the Kronecker delta.  For fixed $\alpha, \beta$, the recurrence relation of Jacobi polynomials is
\begin{align*}
    p_n^{(\alpha, \beta)}(t)=&\frac{2n+\alpha+\beta-1}{2n(n+\alpha+\beta)(2n+\alpha+\beta-2)}\left((2n+\alpha+\beta)(2n+\alpha+\beta-2)t+\alpha^2-\beta^2\right)p_{n-1}^{(\alpha, \beta)}(t)\\
    &-\frac{(n+\alpha-1)(n+\beta-1)(2n+\alpha+\beta)}{n(n+\alpha+\beta)(2n+\alpha+\beta-2)}p_{n-2}^{(\alpha, \beta)}(t).
\end{align*}
\subsubsection*{Hermite polynomials}
Hermite polynomials are a system of orthogonal polynomials with respect to the weight function $w: (-\infty, \infty)\rightarrow \mathbb R$ such that $w(t)=\exp(-{t^2}/{2})$. These are called the probabilist's Hermite polynomials, which we will use throughout the section. There is another form called the physicist's Hermite polynomials with respect to the weight function $w(t)=\exp(-t^2)$. The explicit expression of the probabilist's Hermite polynomials can be written as
\begin{equation*}
    H_n(t)=n!\sum_{m=0}^{\lfloor\frac{n}{2}\rfloor}\left(-\frac{1}{2}\right)^m\frac{t^{n-2m}}{m!(n-2m)!},
\end{equation*}
with the orthogonality property
\begin{equation}\label{hermite_ortho}
    \int_{-\infty}^\infty H_m(t)H_n(t)e^{-\frac{t^2}{2}}dt=\sqrt{2\pi}n!\delta_{mn}.
\end{equation}
Lastly, we state the recurrence relation of Hermite polynomials as $H_{n+1}(t)=xH_n(t)-nH_{n-1}(t)$. Note that the weight of Hermite polynomials can be viewed as an unnormalised normal distribution. If we are more interested in a particular region far away from the origin, we can define a ``shift-and-scale'' version of Hermite polynomials with respect to the weight
\begin{equation*}
    w^{t_0, \epsilon}(t)=\exp((t-t_0)^2/2\epsilon^2),
\end{equation*}
where $t_0$ denotes the new centre, and $\epsilon$ is the standard deviation. Let $(H_n^{t_0, \epsilon})_{n\in\mathbb N}$ denote the shift-and-scale Hermite polynomials. Then, the orthogonality property is
\begin{align*}
    \int_{-\infty}^\infty H^{t_0, \epsilon}_m(t)H^{t_0, \epsilon}_n(t)e^{-\frac{(t-t_0)^2}{2\epsilon^2}}dt=\epsilon\int_{-\infty}^\infty H^{t_0, \epsilon}_m(t_0+\epsilon y)H^{t_0, \epsilon}_n(t_0+\epsilon y)e^{-\frac{y^2}{2}}dy,
\end{align*}
by substitution $y=(t-t_0)/\epsilon$. Hence, if 
\begin{equation}\label{hermite_relation}
    H^{t_0, \epsilon}_n(t_0+\epsilon y)=H_n(y), \qquad n\in\mathbb N,
\end{equation}
then $(H_n^{t_0, \epsilon})_{n\in\mathbb N}$ is an orthogonal polynomial system with orthogonality
\begin{equation*}
    \int_{-\infty}^\infty H^{t_0, \epsilon}_m(t)H^{t_0, \epsilon}_n(t)e^{-\frac{(t-t_0)^2}{2\epsilon^2}}dt=\epsilon\int_{-\infty}^\infty H_m(y)H_n(y)e^{-\frac{y^2}{2}}dy=\epsilon\sqrt{2\pi}n!\delta_{mn},
\end{equation*}
which follows from the orthogonality of Hermite polynomials in Equation (\ref{hermite_ortho}). Similarly, the connection between Hermite and shift-and-scale Hermite polynomials in Equation (\ref{hermite_relation}) provides a way to find the explicit form and recurrence relation of $(H_n^{t_0, \epsilon})_{n\in\mathbb N}$, which are
\begin{align}
    &H^{t_0, \epsilon}_n(t)=n!\sum_{m=0}^{\lfloor\frac{n}{2}\rfloor}\left(-\frac{1}{2}\right)^m\frac{1}{m!(n-2m)!}\left(\frac{t-t_0}{\epsilon}\right)^{n-2m}\label{Hermite_sas_form},\\
    &H^{t_0, \epsilon}_{n+1}(t)=\frac{1}{\epsilon}(t-t_0)H^{t_0, \epsilon}_n(t)-nH^{t_0, \epsilon}_{n-1}(t).\label{Hermite_sas_rec}
\end{align}
\begin{remark}\label{hermite_step}
    Note that there is a simple expression for $(H_n^{t_0, \epsilon})_{n\in\mathbb N}$ at $t=t_0$. One can easily observe that
    \begin{equation*}
        H^{t_0, \epsilon}_n(t_0)=\begin{cases}\left(-\frac{1}{2}\right)^{\frac{n}{2}}\frac{n!}{\frac{n}{2}!}\qquad &\text{for even }n \\
        0\qquad &\text{for odd }n 
        \end{cases}.
    \end{equation*}
\end{remark}

\subsection{Fourier series}
One can also represent a function by a trigonometric series. Here, we only present a brief introduction to the Fourier series, providing complementary details to the main result in \cref{thm:fourier_inversion}.

\subsubsection{Trigonometric series}
Let $f\in L^1(-\pi, \pi)$. The Fourier series of $f$ is defined by
\begin{equation*}
    F(t)=\frac{a_0}{2} + \sum_{k=1}^\infty\left(a_k\cos(kt)+b_k\sin(kt)\right),
\end{equation*}
where
\begin{align*}
    &a_k = \frac{1}{\pi}\int_{-\pi}^\pi f(t)\cos(kt)dx, \qquad k\geq0\\
    &b_k = \frac{1}{\pi}\int_{-\pi}^\pi f(t)\sin(kt)dx, \qquad k\geq1,
\end{align*}
which can be derived from the orthogonal bases $\{\cos{kt}\}_k$ and $\{\sin{kt}\}_k$. More generally, we can extend the period to $2l\in\mathbb R$. For $f\in L^1(-l, l)$ and $k\in \mathbb{Z}$, 
\begin{equation}\label{eq: Fourier_coeff}
    F(t) = \sum_{n=-\infty}^{\infty}c_ke^{i\frac{2\pi}{l}kt},\qquad c_k=\frac{1}{l}\int_{0}^lf(t)e^{-i\frac{2\pi}{l}kt}dt.
\end{equation}

In the setting of the Fourier series, the expression for the $k$-th coefficient $c_k$ in the exponential form can be defined as a linear functional $\mathcal L_k(x)=c_k^x$ on the space of Fourier series, for $x\in L^1(-l, l)$.

\subsubsection{Convergence}
Under some regularity conditions, $F(t)$ converges to $f(t)$ \citep{the_num_ana}. To examine the convergence of the Fourier series, we define the partial sum of the Fourier series as
\begin{equation*}
    S_Nf(t)=\frac{a_0}{2} + \sum_{k=1}^N\left(a_k\cos(kt)+b_k\sin(kt)\right).
\end{equation*}
Now we present pointwise convergence and uniform convergence results \citep{the_num_ana} of the Fourier series for various functions.
\begin{theorem}[Pointwise convergence for bounded variation]\label{2fourier_pointwise_converge}
    For a $2\pi$-periodic function $f$ of bounded variation on $[-\pi, \pi]$, its Fourier series at an arbitrary $t$ converges to 
    \begin{equation*}
        \frac{1}{2}\left(f(t^-)+f(t^+)\right).
    \end{equation*}
\end{theorem}
\begin{theorem}[Uniform convergence for piecewise smooth functions]
    If $f$ is a $2\pi$-periodic piecewise smooth function, 
    \begin{enumerate}[label=(\alph*)]
        \item if $f$ is also continuous, then the Fourier series converges uniformly and continuously to $f$;
        \item if $f$ is not continuous, then the Fourier series converges uniformly to $f$ on every closed interval without discontinuous points.
    \end{enumerate}
\end{theorem}
\begin{theorem}[Uniform error bounds]
    Let $f\in C_p^{k, \alpha}(2\pi)$ be a $2\pi$-periodic $k$-times continuously differentiable function that is Hölder continuous with the exponent $\alpha\in(0, 1]$. Then, the $2$-norm and infinity-norm bound of the partial sum $S_Nf$ can be expressed as
    \begin{equation*}
        \frac{1}{\sqrt{2\pi}}\|f-S_Nf\|_{2}\leq\|f-S_Nf\|_\infty\leq c_k\frac{\log N}{N^{k+\alpha}}, \qquad \text{ for } N\geq2.
    \end{equation*}
\end{theorem}
For functions only defined on an interval $[a, b]$, we can always shift and extend them to be $2\pi$-periodic functions. These theorems guarantee the convergence of common functions we will use in later experiments. To illustrate this, Figure \ref{fig:fourier_converge} compares the convergence theorem bounds to numerical approximation results. Note that compared with the path $x_2(t)=\sin(10t)+e^{2t}-t$, the other 2 paths have better empirical convergence results. The main reason is that the Fourier series of $x_2$ at $t=\pm 1$ does not pointwise converge to $x_2(\pm 1)$. Since the Fourier series treats the interval $[-1, 1]$ as one period over $\mathbb R$, by \cref{2fourier_pointwise_converge}, the series will converge to $(x_2(-1)+x_2(1))/2$ at $t=\pm 1$, leading to incorrect convergence at boundaries. This property also arises in real-world non-periodic time series, leading us to introduce the \textit{mirror augmentation} later in \cref{remark: mirror-trick}.

Comparing Figures \ref{fig:ortho_converge} and \ref{fig:fourier_converge}, one can observe that orthogonal polynomials are better at approximating continuously differentiable paths, while the Fourier series is better at estimating paths with spikes, and its computation is more stable in the long run. In a later section, Figure \ref{fig:orth_approx} provides a summary of convergence results for different types of orthogonal polynomials and Fourier series, which also match the results shown here.

\begin{figure}[htbp]
    \centering
    \includegraphics[width=\linewidth]{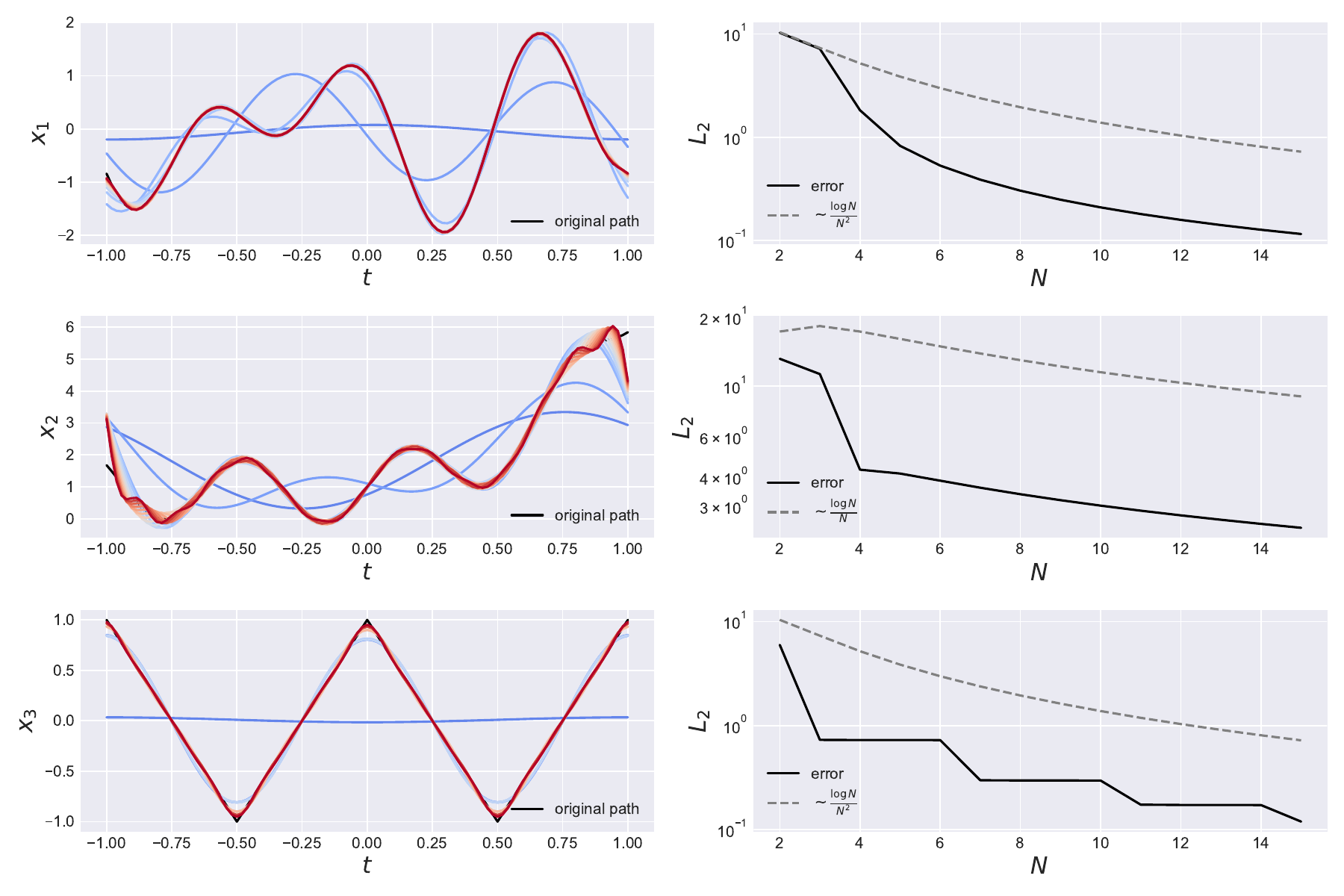}
    \caption{Approximation (left) and $L_2$ convergence (right) results for Fourier series by increasing order $N$, with the same experimental setting as Figure \ref{fig:ortho_converge}.}
    \label{fig:fourier_converge}
\end{figure}

\subsection{Approximation quality of orthogonal polynomials and Fourier series}
Finally, we present a numerical comparison of the approximation results given by the methods introduced above. 

In particular, we will visualise the approximation results of the following objects:
\begin{itemize}
    \item Legendre polynomials: $w(t)=1$
    \item two types of Jacobi polynomials: $w(t)=\sqrt{1+t}$, $w(t)=\sqrt{1-t}$
    \item three types of shift-and-scale Hermite polynomials with different variance for \textit{pointwise approximation}: $\epsilon=0.1, 0.05, 0.01$
    \item Fourier series
\end{itemize}
For pointwise approximation via the Hermite polynomials, each sample point $t_i$ of the function will be approximated by a system of Hermite polynomials centred at the point $t_i$, i.e., $(H_n^{t_i, \epsilon})_{n\in\mathbb{N}}$.
To test the approximation quality, we simulate random polynomial functions and random trigonometric functions. The $L_2$ error is then obtained by an average over $10$ functions of each type.

\subsubsection{Approximation results}
Figure \ref{fig:orth_approx} demonstrates the reduction in $L_2$ error as the order of orthogonal polynomials and Fourier series increases. Among all the bases considered, the Fourier series provides the least accurate approximation for both path types. This is due to its inability to guarantee pointwise convergence at $\pm 1$, which is caused by boundary inconsistencies. Three types of Jacobi polynomials, including Legendre polynomials, show comparable approximation results, with a slight edge in convergence observed for Legendre polynomials. On the other hand, Hermite polynomials exhibit a much lower approximation error due to their shifting focus on the point of interest. However, decreasing $\epsilon$ to sharpen the focus on sample points can cause the coefficients of Hermite polynomials to inflate rapidly. This behaviour is consistent with the analytic form and recurrence relation of the shift-and-scale Hermite polynomials, as detailed in Equation (\ref{Hermite_sas_form}) and Equation (\ref{Hermite_sas_rec}). As a result, Hermite polynomials with $\epsilon=0.01$ do not outperform those with $\epsilon=0.05$. The step-like pattern of decrease observed in Hermite polynomials can be traced back to Remark \ref{hermite_step}.
Figure \ref{fig:orth_approx_real} shows the $L_2$ approximation error of different bases for two of our real-world datasets. Here, we see that apart from using shift-and-scale Hermite polynomials which come with additional complexity (see Figure \ref{fig:inv_error}), the Fourier series seems to be the best candidate for approximation.

\begin{figure}[htbp]
    \centering
    \includegraphics[width=0.9\linewidth]{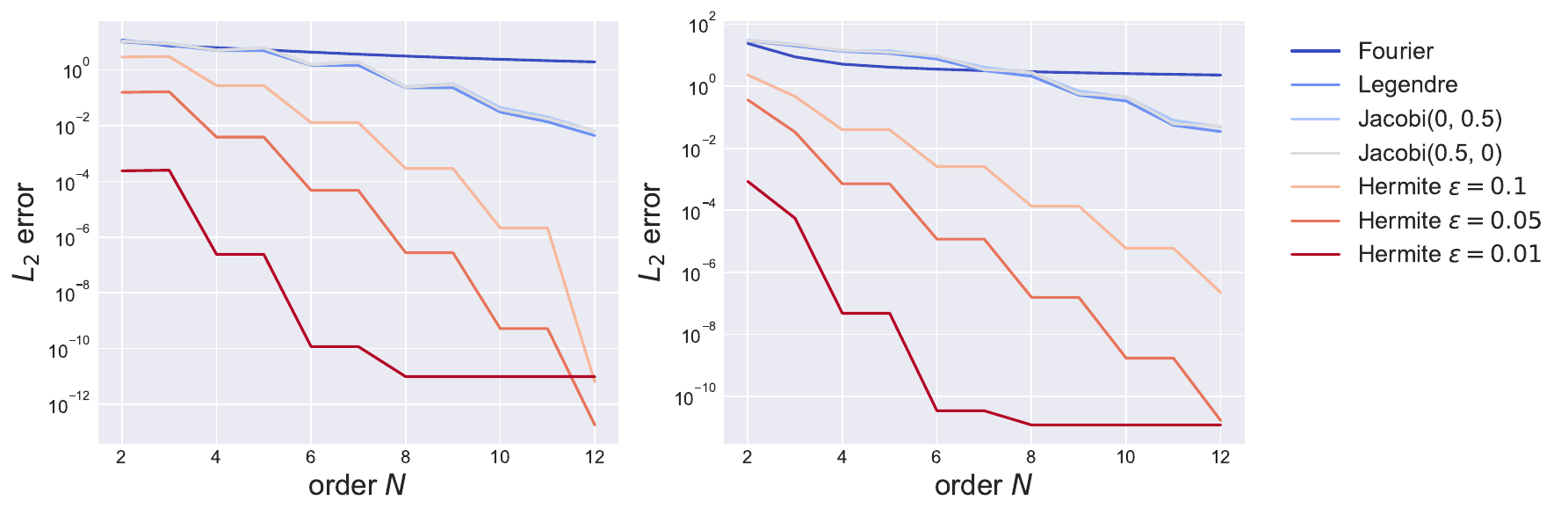}
    \caption{$L_2$ approximation error using different bases. The figures (from left to right) are the corresponding error averaged over $10$ random polynomial and trigonometric functions.}
    \label{fig:orth_approx}
\end{figure}

\begin{figure}[htbp]
    \centering
    \includegraphics[width=0.9\linewidth]{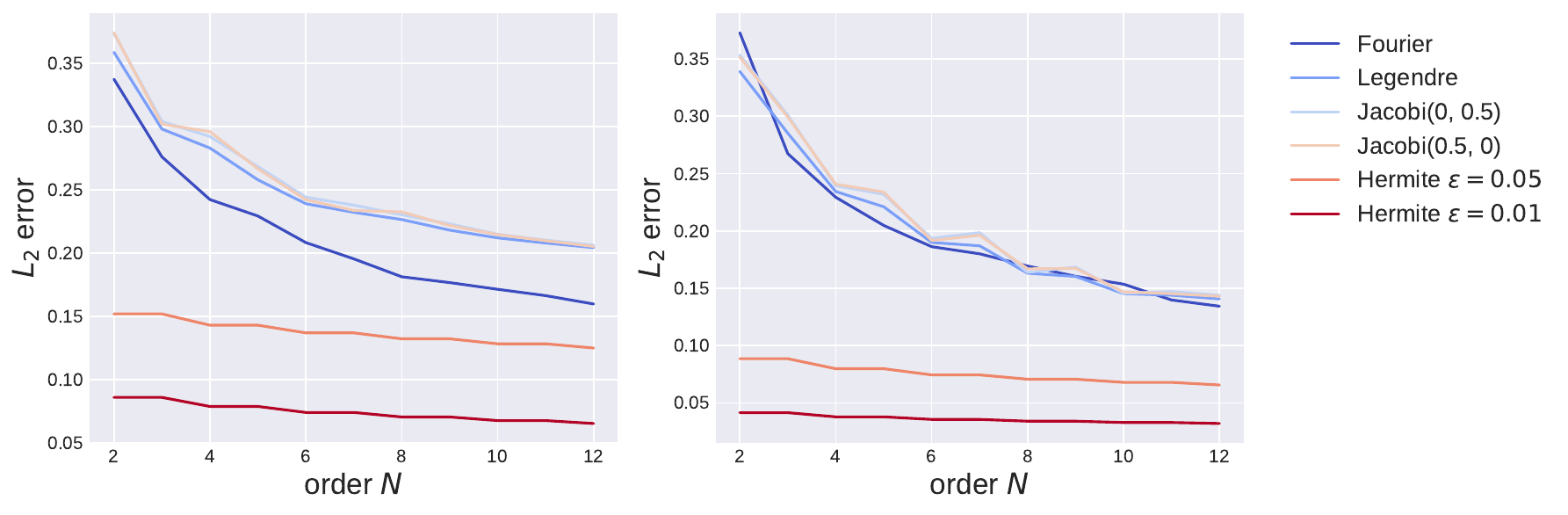}
    \caption{Real data $L_2$ approximation error using different bases. The figures (from left to right) are the corresponding error averaged over $10$ random samples from the HEPC and Exchange rates datasets.}
    \label{fig:orth_approx_real}
\end{figure}

To mitigate computational expense, we henceforth use Hermite polynomials with $\epsilon=0.05$ as the representative of the Hermite family. The findings presented in Figures \ref{fig:orth_approx} and \ref{fig:orth_approx_real} play a crucial role in our signature inversion method, as they establish a benchmark for the best possible performance attainable in path reconstruction from signatures.

\section{Proofs of Signature Inversion }\label{sec:proofs}
In this section, we present the formal proofs of the signature inversion \cref{thm:poly_inversion}
 and \cref{thm:fourier_inversion}, along with the remark in \cref{remark: taylor} about Taylor approximation of the weight function.
\subsection{Proof of orthogonal polynomial inversion \cref{thm:poly_inversion}}\label{appendix:poly_proof}
Recall the statement in \cref{thm:poly_inversion} deriving the $n$-th polynomial coefficient $\alpha_n$ (see Equation (\ref{eqn:coeff-poly})) via a recurrence relation:

Let $x : [a,b] \to \mathbb R$ be a smooth path such that  $x(a)=0$. Consider the augmentation $\hat x(t) = (t, \omega(t)x(t)) \in \mathbb R^2$, where $\omega(t)$ corresponds to the weight function of a system of orthogonal polynomials $(p_n)_{n\in\mathbb{N}}$, and is well defined on the closed and compact interval $[a,b]$. Then, there exists a linear combination $\ell_n$ of words such that the $n^{th}$ coefficient in Equation (\ref{eqn:coeff-poly}) satisfies $\alpha_n=\langle\ell_n, S(\hat x)\rangle$. Furthermore, the sequence $(\ell_n)_{n \in \mathbb N}$ satisfies the following recurrence relation
\begin{align*}
    \ell_{n} = A_n\frac{(p_{n-1}, p_{n-1})}{(p_n, p_n)}e_1\succ \ell_{n-1} + (A_na+B_n)\frac{(p_{n-1}, p_{n-1})}{(p_n, p_n)}\ell_{n-1} + C_n\frac{(p_{n-2}, p_{n-2})}{(p_n, p_n)}\ell_{n-2},
\end{align*}
with 
\begin{align*}
    \ell_0=\frac {A_0}{(p_0, p_0)}e_{21}\quad \text{ and }\quad\ell_1=\frac {A_1}{(p_1, p_1)}(e_{121} + e_{211})+\frac {A_1a+B_1}{(p_1, p_1)}e_{21}.
\end{align*}

\begin{proof}
One can express the first two coefficients in an orthogonal polynomial expansion of $x$ by the signature:
\begin{align*}
    \alpha_0 &= \frac 1{( p_0, p_0)} \int_{a}^{b}A_0x(t)\omega(t)dt\\
    &=\langle\frac {A_0}{( p_0, p_0)} e_2\succ e_1, S(\hat x)\rangle\\
    &=\langle\frac {A_0}{( p_0, p_0)}e_{21}, S(\hat x)\rangle\\
    &=\langle\ell_0, S(\hat x)\rangle,\\
    \alpha_1 &= \frac 1{( p_1, p_1)} \int_{a}^{b}(A_1t+B_1)x(t)\omega(t)dt\\
    &=\frac {A_1}{( p_1, p_1)} \int_{a}^{b} (t-a)x(t)\omega(t)dt+\frac {A_1a+B_1}{( p_1, p_1)} \int_{a}^{b} x(t)\omega(t)dt\\
    &=\frac {A_1}{( p_1, p_1)} \langle(e_1 \shuffle e_2)\succ e_1, S(\hat x)\rangle+\frac {A_1a+B_1}{( p_1, p_1)} \langle e_{21}, S(\hat x)\rangle\\
    &=\langle\frac {A_1}{( p_1, p_1)}(e_{121}+e_{211})+\frac {A_1a+B_1}{( p_1, p_1)} e_{21}, S(\hat x)\rangle\\
    &=\langle\ell_1, S(\hat x)\rangle.
\end{align*}
Then one can find $\ell_n$ recursively by multiplying both sides of Equation (\ref{eq:general_rec}) by $x(t)\omega(t)$ and integrating on $[a, b]$:

\begin{align*}
    \int_{a}^{b}p_{n}(t)x(t)\omega(t)dt=&\int_{a}^{b}(A_nt+B_n)p_{n-1}(t)x(t)\omega(t)dt + \int_{a}^{b}C_np_{n-2}(x)x(t)\omega(t)dt\\
    =&A_n\int_{a}^{b}(t-a)d\left(\int_{a}^tp_{n-1}(s)x(s)\omega(s)ds\right)\\
    &+ (A_na+B_n)\int_{a}^{b}p_{n-1}(t)x(t)\omega(t)dt \\
    &+C_n\int_{a}^{b}p_{n-2}(x)x(t)\omega(t)dt.
\end{align*}
By definition of $\alpha_n$,
\begin{align*}
    \int_{a}^{b}p_{n}(t)x(t)\omega(t)d&=(p_n, p_n)\alpha_{n}=(p_n, p_n)\langle\ell_{n}, S(\hat x)\rangle,\\
    A_n\int_{a}^{b}(t-a)d\left(\int_{a}^tp_{n-1}(s)x(s)\omega(s)ds\right)&=A_n(p_{n-1}, p_{n-1})\langle e_1\succ \ell_{n-1}, S(\hat x)\rangle,\\
    (A_na+B_n)\int_{a}^{b}p_{n-1}(t)x(t)\omega(t)dt&=(A_na+B_n)(p_{n-1}, p_{n-1})\langle\ell_{n-1}, S(\hat x)\rangle,\\
    C_n\int_{a}^{b}p_{n-2}(x)x(t)\omega(t)dt&=C_n( p_{n-2}, p_{n-2})\langle\ell_{n-2}, S(\hat x)\rangle.
\end{align*}
Therefore, the recurrence relation of the linear functions on the signature retrieving the coefficients of orthogonal polynomials is
\begin{align*}
    \langle\ell_{n}, S(\hat x)\rangle =& A_n\frac{( p_{n-1}, p_{n-1})}{( p_n, p_n)}\langle e_1\succ \ell_{n-1}, S(\hat x)\rangle \\
    &+ (A_na+B_n)\frac{( p_{n-1}, p_{n-1})}{( p_n, p_n)}\langle\ell_{n-1}, S(\hat x)\rangle \\
    &+ C_n\frac{( p_{n-2}, p_{n-2})}{( p_n, p_n)}\langle\ell_{n-2}, S(\hat x)\rangle.
\end{align*}
\end{proof}
Several assumptions are made in order to derive the recurrence relation. Namely, the interval defined on the inner space must be compact, and the weight function $\omega(t)$ should be well-defined on the closed interval $[a,b]$. These assumptions can limit the range of applicable orthogonal polynomial families. For example, since the range of Hermite polynomials is unbounded, they are not suitable for our inversion. However, by using a shift-and-scale version of these polynomials, where most of the weight density is concentrated at a particular point, their weight can be numerically truncated to a compact interval. One can centre the weight density using a small enough $\epsilon$ and shift it to a point of interest $t_i$. Since the non-zero density region is now concentrated in a small interval, \cref{thm:poly_inversion} can be applied to the truncated density over this interval. The relationship between the original Hermite polynomials and the shift-and-scale Hermite polynomials is given by Equation (\ref{hermite_relation}).

\subsection{Remark on the Taylor approximation of the weight function}\label{appendix:taylor}

The results in \cref{thm:poly_inversion} require signatures of $\hat x = (t, w(t)x(t))$. However, sometimes one may only have signatures of $\tilde x = (t, x(t))$. Here, we propose a theoretically applicable method by approximating the weight function as a Taylor polynomial. 

Consider the Taylor approximation of $\omega$ around $t=a$, i.e.,
$$\omega(t) \approx \sum_{i=0}^M \frac{d^i\omega}{dt^i}\Big|_{t=a} (t-a)^i=\sum_{i=0}^M \omega_i (t-a)^i.$$
Letting $\tilde x_t = (t, x(t))$ and
$$c_i := (e_2 \shuffle e_1^{\shuffle i}) \succ e_1 = i!(e_{21\ldots 1} + e_{121\ldots1} + ... + e_{1\ldots121}),$$
we have
{\allowdisplaybreaks
\begin{align*}
    \alpha_0 &= \frac 1{( p_0, p_0)} \int_a^bA_0x(t)\omega(t)dt \\
    &= \frac {A_0}{( p_0, p_0)} \sum_{i=0}^M\omega_i\int_a^b(t-a)^ix(t)dt \\
    &=\langle\frac {A_0}{( p_0, p_0)}\sum_{i=0}^M \omega_i  (e_2 \shuffle e_1^{\shuffle i}) \succ e_1, S(\tilde x)\rangle\\
    &=\langle\frac {A_0}{( p_0, p_0)}\sum_{i=0}^M \omega_i  c_i, S(\tilde x)\rangle\\
    &= \langle\ell_0, S(\tilde x)\rangle,\\
    \alpha_1 &= \frac 1{( p_1, p_1)} \int_a^b(A_1t+B_1)x(t)\omega(t)d t\\
    &=\frac {1}{( p_1, p_1)} \int_a^b \left(A_1(t-a) + A_1a+B_1\right)x(t)\omega(t)d t\\
    &=\frac {1}{( p_1, p_1)} \sum_{i=0}^M  \omega_i \int_a^b \left(A_1(t-a)^{i+1} + (A_1a+B_1)(t-a)^i\right)x(t)d t\\
    &= \langle\frac{1}{( p_1, p_1)} \sum_{i=0}^M\omega_i \left(A_1(e_2 \shuffle e_1^{\shuffle i+1}) + (A_1a+B_1)(e_2 \shuffle e_1^{\shuffle i})\right) \succ e_1, S(\tilde x)\rangle\\
    &= \langle\frac{1}{( p_1, p_1)} \sum_{i=0}^M\omega_i (A_1 c_{i+1} + (A_1a+B_1)c_i), S(\tilde x)_{a,b}\rangle\\
    &=\langle\ell_1, S(\tilde x)\rangle\\
\end{align*}}
By induction, the same relation as \cref{thm:poly_inversion} holds
$$\ell_{n} = A_n\frac{( p_{n-1}, p_{n-1})}{( p_n, p_n)} e_1 \succ \ell_{n-1} + (A_na+B_n)\frac{( p_{n-1}, p_{n-1})}{( p_n, p_n)}\ell_{n-1}+ C_n\frac{( p_{n-2}, p_{n-2})}{( p_n, p_n)}\ell_{n-2}.$$

There are several reasons why we consider the Taylor approximation method ``theoretically applicable.'' The expansion of the weight function around a point $a$ can be hard to find analytically, and even if the series is found, it may diverge. If the series does converge, we still need to determine the number of terms required to meet a certain error tolerance. Furthermore, if the convergence rate is slow, more terms are needed in the series. This leads to higher necessary signature truncation levels and consequently also higher computational complexity.

\subsection{Proof of Fourier inversion in \cref{thm:fourier_inversion}}\label{appendix:fourier_proof}
Recall the statement in \cref{thm:fourier_inversion} deriving the Fourier coefficients $a_0, a_n, b_n$ (see Equations (\ref{eqn: a0}), (\ref{eqn: an}), (\ref{eqn: bn})) of a path as follows:

Let $x : [0,2\pi] \to \mathbb R$ be a periodic smooth path such that $x(0)=0$, and consider the augmentation $\hat{x}(t) = (t, \sin(t), \cos(t)-1, x(t)) \in \mathbb R^4$. Then the following relations hold
\begin{equation*}
    \begin{split}
    a_0 &= \frac{1}{2\pi}\langle e_4 \succ e_1, S(\hat{x}) \rangle, \\
    a_n &= \frac{1}{\pi}\sum_{k=0}^n \sum_{q=0}^k \binom{n}{k}\binom{k}{q} \cos(\frac{1}{2}(n-k)\pi) \langle (e_4\shuffle e_2^{\shuffle n-k} \shuffle e_3^{\shuffle q} ) \succ e_1, S(\hat{x}) \rangle,\\
    b_n &= \frac{1}{\pi}\sum_{k=0}^n \sum_{q=0}^k \binom{n}{k}\binom{k}{q} \sin(\frac{1}{2}(n-k)\pi) \langle (e_4\shuffle e_2^{\shuffle n-k} \shuffle e_3^{\shuffle q} ) \succ e_1, S(\hat{x})\rangle.
    \end{split}
\end{equation*}

\begin{proof}
    
By Multiple-Angle formulas, we have
\begin{equation}\label{eqn: sin_multiple}
    \sin(nt) = \sum_{k=0}^n \binom{n}{k}\cos^k(t)\sin^{n-k}(t)\sin(\frac{1}{2}(n-k)\pi),
\end{equation}
\begin{equation}\label{eqn: cos_multiple}
    \cos(nt) = \sum_{k=0}^n \binom{n}{k}\cos^k(t)\sin^{n-k}(t)\cos(\frac{1}{2}(n-k)\pi).
\end{equation}

We can now connect Equation (\ref{eqn: bn}), Equation (\ref{eqn: sin_multiple}), and the shuffle identity of the signature described in \cref{thm:shuffle} to obtain an expression for $b_n$ as 
\begin{equation}
    \begin{split}
    b_n &= \frac{1}{\pi}\int_0^{2\pi} x(t)\sin(nt)dt \\
        &= \frac{1}{\pi}\int_0^{2\pi} x(t) \sum_{k=0}^n \binom{n}{k}\cos^k(t)\sin^{n-k}(t)\sin(\frac{1}{2}(n-k)\pi)dt\\
        &= \frac{1}{\pi}\sum_{k=0}^n \binom{n}{k}\sin(\frac{1}{2}(n-k)\pi)\int_0^{2\pi} x(t) \cos^k(t)\sin^{n-k}(t)dt\\
        &= \frac{1}{\pi}\sum_{k=0}^n \binom{n}{k}\sin(\frac{1}{2}(n-k)\pi)\int_0^{2\pi} x(t) ((\cos(t)-1)+1)^k \sin^{n-k}(t)dt\\
        &= \frac{1}{\pi}\sum_{k=0}^n \binom{n}{k}\sin(\frac{1}{2}(n-k)\pi)\int_0^{2\pi} x(t) \sin^{n-k}(t) \sum_{q=0}^k \binom{k}{q}(\cos(t)-1)^q dt\\
        &= \frac{1}{\pi}\sum_{k=0}^n \sum_{q=0}^k \binom{n}{k}\binom{k}{q} \sin(\frac{1}{2}(n-k)\pi) \int_0^{2\pi} x(t) \sin^{n-k}(t) (\cos(t)-1)^q dt\\
        &= \frac{1}{\pi}\sum_{k=0}^n \sum_{q=0}^k \binom{n}{k}\binom{k}{q} \sin(\frac{1}{2}(n-k)\pi) \langle (e_4\shuffle e_2^{\shuffle n-k} \shuffle e_3^{\shuffle q} ) \succ e_1, S(\hat{x}) \rangle.\\
    \end{split}
\end{equation}

Similarly, we rearrange Equation (\ref{eqn: an}) with Equation (\ref{eqn: cos_multiple}) to obtain the formula for $a_n$
\begin{equation}
    \begin{split}
    a_n &= \frac{1}{\pi}\int_0^{2\pi} x(t)\cos(nt)dt \\
        &= \frac{1}{\pi}\sum_{k=0}^n \sum_{q=0}^k \binom{n}{k}\binom{k}{q} \cos(\frac{1}{2}(n-k)\pi) \langle (e_4\shuffle e_2^{\shuffle n-k} \shuffle e_3^{\shuffle q} ) \succ e_1, S(\hat{x}) \rangle.\\
    \end{split}
\end{equation}

Finally, we get $a_0$ immediately as
\begin{equation}
    a_0 = \frac{1}{2\pi}\int_0^{2\pi} x(t)dt = \langle e_4 \succ e_1, S(\hat{x}) \rangle.
\end{equation}
\end{proof}

\section{Visualising Inversion by Different Bases}\label{appendix:visualisation}
We demonstrate the quality of signature inversion on paths generated from fractional Brownian motion \citep{mandelbrot1968fractional} (see \cref{fig:visual_sig_inv_rough}), and two of our real-world datasets (see \cref{fig:visual_sig_inv_rough_real}) using five different polynomial and Fourier bases. The orders of the bases recovered in these figures are listed in Table \ref{table:inv_orders}, establishing a relationship with the levels of truncated signatures. The reconstruction results are influenced by three main factors:
\begin{enumerate} 
    \item the degree/order of the bases, $n$, and the corresponding levels of truncated signatures; \label{fac0} 
    \item the complexity of the paths, such as frequency and smoothness; \label{fac1} 
    \item the weight function of the orthogonal polynomials. \label{fac3} 
\end{enumerate}
Factor \ref{fac0} plays a crucial role in path reconstruction by providing a theoretical bound on the inversion error, as the order of the retrieved polynomial and Fourier coefficients is constrained by the truncation level of the signature.  Recall Figure \ref{fig:orth_approx}, which demonstrates that increasing the basis order improves the approximation quality, implying that reconstructions from higher-level signatures will more closely approximate the original path.

Factor \ref{fac1} also plays a crucial role in the quality of approximation. A comparison between the first and second columns of Figure \ref{fig:visual_sig_inv_rough} shows that all bases can approximate the simpler path in the second column more accurately. Consequently, more complex paths may lead to less satisfactory inversion results due to the limitations of the bases, as previously discussed in \cref{section:tradeoff}.

Compared to the previous factors, factor \ref{fac3} has a relatively minor impact on the reconstruction process. As shown in Figure \ref{fig:visual_sig_inv_rough_real}, the left tail of the Jacobi$(0, 0.5)$ approximation and the right tail of the Jacobi$(0.5, 0)$ approximation tend to diverge as their weight functions approach zero at $t\rightarrow\pm 1$. Meanwhile, the signature inversion using Hermite polynomials, performed on a pointwise basis, remains precise even at lower polynomial degrees, as each sample point is estimated at the centre of the weight function.

\begin{table}\label{table:inv_orders}
  \caption{Polynomial/Fourier basis orders used in Figures \ref{fig:visual_sig_inv_rough} and \ref{fig:visual_sig_inv_rough_real}.}
  \centering
  \begin{tabular}{lll}
    \toprule
    \multicolumn{1}{l}{Approximation method}& \multicolumn{1}{l}{Order $n$}&\multicolumn{1}{l}{Level of truncated signature}\\
    \midrule
    Legendre & 10 & n+2=12\\
    Jacobi$(0, 0.5)$& 10 & n+2=12\\
    Jacobi$(0.5, 0)$& 10 & n+2=12\\
    Hermite $(\epsilon=0.05)$ & 2 & n+2=4\\
    Fourier & 10 & n+2=12 \\
    \bottomrule
  \end{tabular}
\end{table}

\begin{figure}[hbtp]
    \centering
    \includegraphics[width=\linewidth]{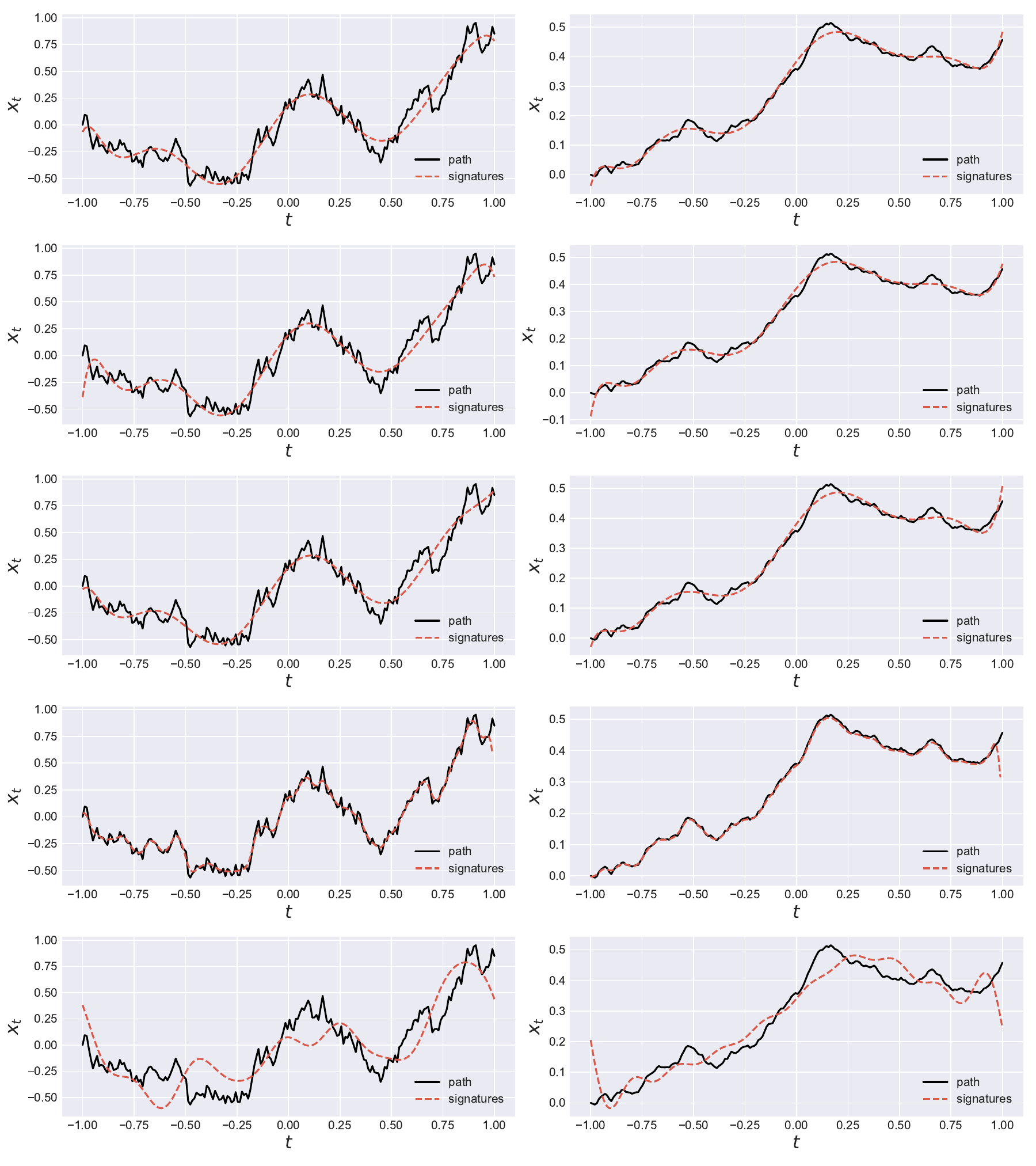}
    \caption{Inversion results on fractional Brownian motion with Hurst $0.5$ and $0.9$, with approximation bases (from top to bottom) Legendre (Jacobi$(0, 0)$), Jacobi$(0, 0.5)$, Jacobi$(0.5, 0)$, Hermite $(\epsilon=0.05)$ and Fourier.}
    \label{fig:visual_sig_inv_rough}
\end{figure}

\begin{figure}[hbtp]
    \centering
    \includegraphics[width=\linewidth]{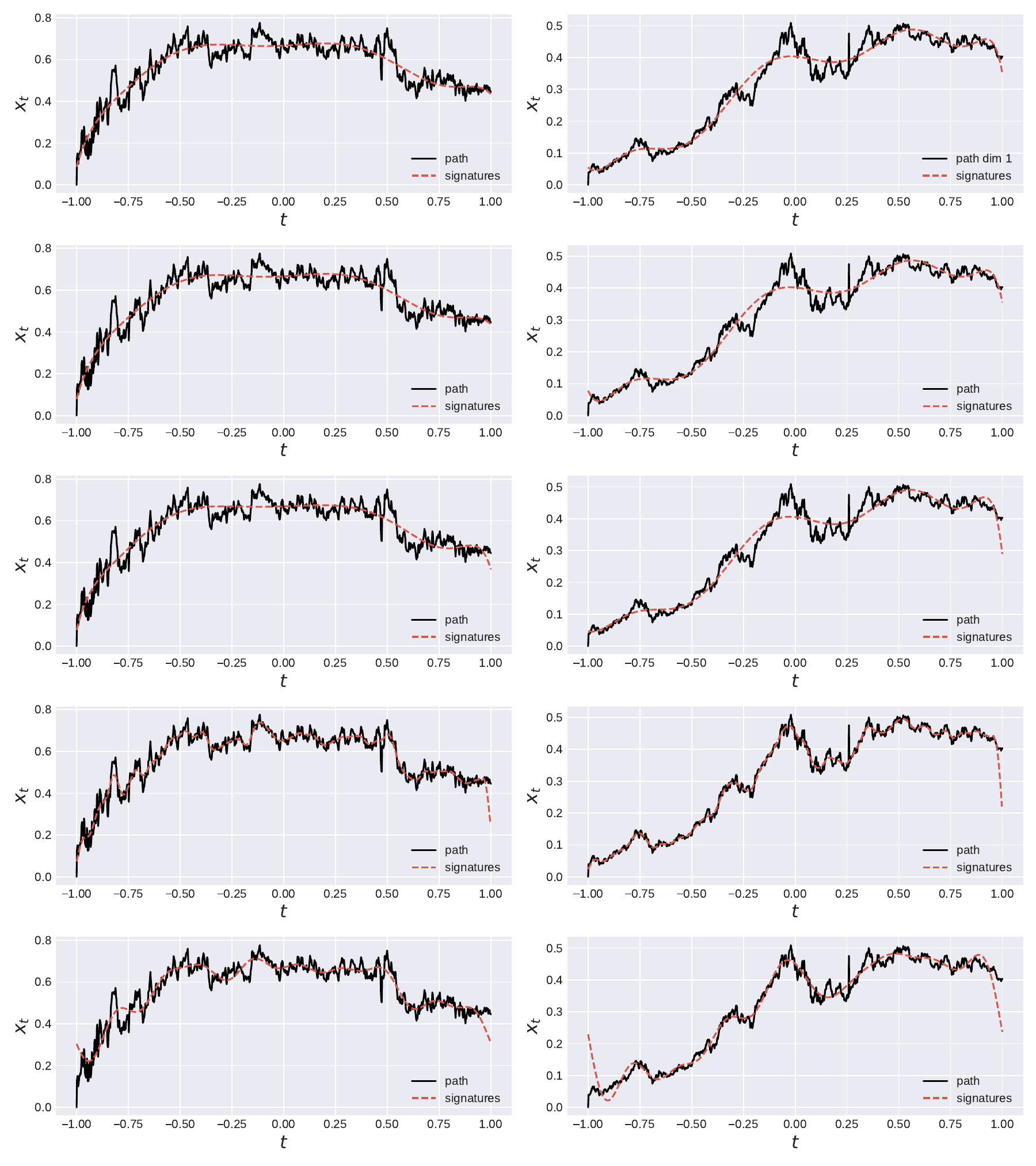}
    \caption{Inversion results on real-world time series. The left column is a sample from the HEPC dataset. The right column is a sample from the Exchange rates dataset (for readability, we only plot one of the eight dimensions). The approximation bases (from top to bottom) are Legendre (Jacobi$(0, 0)$), Jacobi$(0, 0.5)$, Jacobi$(0.5, 0)$, Hermite $(\epsilon=0.05)$ and Fourier.}
    \label{fig:visual_sig_inv_rough_real}
\end{figure}

\section{Experiment Details}\label{appendix:experiments}
In this section, we provide additional details about the experimental setup. The accompanying code can be found at \texttt{https://github.com/Barb0ra/SigDiffusions}.

We follow the score-based generative diffusion via a variance-preserving SDE paradigm proposed in \citet{song2019generative}. We tune $\overline{\beta}_{min}$ and $\overline{\beta}_{max}$ in Equation (\ref{eq:ODE}) to be 0.1 and 5 respectively. We use a denoising score-matching \citep{vincent2011connection} objective for training the score network $s_\theta$. For sampling, we discretise the probability flow ODE 
\begin{equation}
d\textbf{x}_t = -\frac{1}{2}\beta(t)[\textbf{x}_t+s_\theta(t, \textbf{x}_t)]dt, t\in[0,1]
\end{equation}
with an initial point $\textbf{x}_1\sim \mathcal{N}(0,I)$. To solve the discretised ODE, we use a \texttt{Tsit5} solver with $128$ time steps. We adopt the implementation of the Predictive and Discriminative Score metrics from TimeGAN \citep{yoon2024timegan}.
To satisfy the conditions for Fourier inversion, we augment the paths with additional channels as described in \cref{thm:fourier_inversion}, and we add an extra point to the beginning of each path, making it start with 0.

The model architecture remains fixed throughout the experiments as a transformer with 4 residual layers, a hidden size of 64, and 4 attention heads. Note that other relevant works \citep{yuan2024diffusion, coletta2024constrained, bilovs2023modeling} follow a very similar or bigger architecture. We use the Adam optimiser. We run the experiments on an NVIDIA GeForce RTX 4070 Ti GPU. Table \ref{table: marginals} details additional KS test performance metrics (see \cref{experiments}).

For the task of generating time series described in \cref{subsection:smooth_eval}, we fix the number of samples to 1000, the batch size to 128, the number of epochs to 1200, and the learning rate to 0.001. The details variable across datasets are listed in \cref{table:long-gen}. We reserve $1000$ points from each dataset as an unseen test set for metric calculation.

\begin{table}
  \caption{KS Test average scores and type I errors on the marginals of time series of length 1000.}
  \label{table: marginals}
  \centering
  \resizebox{\columnwidth}{!}{%
  \begin{tabular}{llllll}
    \toprule
     Dataset & Model & t=300 & t=500 & t=700 & t=900 \\
    \midrule
   \multirow{4}{*}{Sines} & SigDiffusion (ours) & \underline{0.25, 50\%} & \underline{0.22, 31\%} & \underline{0.24, 40\%} & \underline{0.23, 38\%} \\
   & DDO ($\gamma = 1$) & \textbf{0.17, 7\%} & \textbf{0.16, 5\%} & \textbf{0.20, 15\% }& \textbf{0.23, 33\%} \\
   & Diffusion-TS & 0.77, 100\% & 0.50, 100\%  & 0.48, 100\% & 0.48, 100\% \\
   & CSPD-GP (RNN)  & 0.62, 100\% & 0.45, 82\% & 0.48, 95\% & 0.55, 99\%\\
   & CSPD-GP (Transformer) & 0.57, 100\% & 0.49, 100\% & 0.51, 99\% & 0.60, 100\% \\
    \midrule
   \multirow{4}{*}{Predator-prey} & SigDiffusion (ours) & \textbf{0.20, 20\%} & \underline{0.29, 76\%} & \textbf{0.23, 39\%} & \textbf{0.22, 35\%} \\
   & DDO ($\gamma = 10$) & 0.34, 92\% & 0.30, 85\% & 0.36, 96\% & 0.40, 100\% \\
   & Diffusion-TS & 1.00, 100\% & 1.00, 100\% & 1.00, 100\% & 0.99, 100\% \\
   & CSPD-GP (RNN) & \underline{0.27, 56\%} & \textbf{0.25, 47\%} & \underline{0.32, 74\%}  & \underline{0.31, 79\%}  \\
   & CSPD-GP (Transformer) & 0.79, 100\% & 0.78, 100\% & 0.79, 100\% & 0.74, 100\%\\
    \midrule
   \multirow{4}{*}{HEPC} & SigDiffusion (ours) & \textbf{0.20, 16\%} & \textbf{0.18, 9\%} & \textbf{0.19, 12\%} & \textbf{0.21, 22\%} \\
   & DDO ($\gamma = 1$) & \underline{0.25, 46\%} & \underline{0.23, 38\%} & \underline{0.25, 46\%} & \underline{0.25, 51\%} \\
   & Diffusion-TS & 0.85, 100\% & 0.87, 100\% & 0.82, 100\% & 0.88, 100\% \\
   & CSPD-GP (RNN) & 0.52, 100\% & 0.53, 100\% & 0.55, 100\% & 0.56, 100\% \\
   & CSPD-GP (Transformer)  & 1.00, 100\% & 1.00, 100\% & 1.00, 100\% & 1.00, 100\% \\
    \midrule
       \multirow{4}{*}{Exchange rates} & SigDiffusion (ours) & \underline{0.31, 80\%} & \underline{0.28, 67\%} & \underline{0.28, 65\%} & \underline{0.31, 74\%}\\
    & DDO ($\gamma = 1$) & \textbf{0.24, 41\%} & \textbf{0.24, 42\%} & \textbf{0.25, 45\%} & \textbf{0.25, 45\%} \\
   & Diffusion-TS & 0.72, 100\% & 0.71, 100\% & 0.70, 100\% & 0.69, 100\%\\
   & CSPD-GP (RNN) & 0.59, 100\% & 0.56, 100\% & 0.55, 100\% & 0.56, 100\%\\
   & CSPD-GP (Transformer) & 1.00, 100\% & 0.99, 100\% & 0.98, 100\% & 0.99, 100\% \\
    \midrule
       \multirow{4}{*}{Weather} & SigDiffusion (ours) & \underline{0.35, 82\%} & \underline{0.34, 80\%} & \underline{0.33, 76\%} & \underline{0.34, 78\%} \\
    & DDO ($\gamma = 10$) & \textbf{0.26, 45\%} & \textbf{0.27, 52\%} & \textbf{0.26, 46\%} & \textbf{0.26, 47\%} \\
   & Diffusion-TS & 0.49, 100\% & 0.50, 100\% & 0.50, 100\% & 0.49, 100\% \\
   & CSPD-GP (RNN) & 0.57, 100\% & 0.56, 100\& & 0.56, 100\% & 0.56, 100\%\\
   & CSPD-GP (Transformer) & 0.91, 100\% &  0.92, 100\% & 0.91, 100\% & 0.91, 100\%\\
    \bottomrule
  \end{tabular}
  }
\end{table}

\paragraph{The mirror augmentation}\label{remark: mirror-trick} For many datasets, we might not wish to assume path periodicity as required in the Fourier inversion conditions in \cref{thm:fourier_inversion}. We observed that a useful trick in this case is to concatenate the path with a reversed version of itself before performing the additional augmentations. We denote this the \textit{mirror augmentation}. Table \ref{table:long-gen} indicates the datasets for which this augmentation is performed.

\begin{table}
  \caption{Datasets for long time series generation.}
  \label{table:long-gen}
  \centering
  \begin{tabular}{lllllll}
    \toprule
    Dataset &  Mirror augmentation & Data points & Dimensions\\
    \midrule
   Sines  & Yes & 10000 & 5 \\
   HEPC & No & 10242 & 1 \\
   Predator-prey & No & 10000 & 2 \\
   Exchange rates & No & 6588 & 8 \\
   Weather & No & 10340 & 14 \\
    \bottomrule
  \end{tabular}
\end{table}

\paragraph{Datasets} As previously described in \cref{experiments}, we measure the performance of \texttt{SigDiffusions} on two synthetic (Sines and Predator-prey) and three real-world (HEPC, Exchange Rates, Weather) public datasets. We generate Sines the same way the Sine dataset is generated in TimeGAN \citep{yoon2019time}, by sampling sine curves at a random phase and frequency but changing the sampling rate to 1000. Predator-prey is a dataset consisting of sample trajectories of a two-dimensional system of ODEs adopted from \citet{bilovs2023modeling}
\begin{equation}
\begin{split}
    \dot{x}=\frac{2}{3}x - \frac{2}{3}xy,\\
    \dot{y}=xy-y.
\end{split}
\label{eqn:lo-vol}
\end{equation}
We generate Predator-prey on a time grid of 1000 points on the interval $t\in[0,10]$. HEPC \citep{uciml2024electric} is a household electricity consumption dataset collected minute-wise for 47 months from 2006 to 2010. We slice the dataset to windows of length 1000 with a stride of 200, yielding a dataset of 10242 entries. We select the \textit{voltage} feature to generate as a univariate time series. We use the Exchange Rates dataset provided in \citet{lai2018modeling, 2017laiguokun} and slice it with a stride of 1, yielding 6588 time series. Lastly, the Weather dataset was measured and published by the Max-Planck-Institute for Biogeochemistry \citep{weather}. We take the first 14 features from this dataset describing the pressure, temperature, humidity, and wind conditions, and we slice the time series with a stride of 5 to get 10340 samples.

\paragraph{Benchmarks}
We compare our models to four recent diffusion models for long time series generation: Diffusion-TS \citep{yuan2024diffusion}, DDO \citep{lim2023score}, and two variants of CSPD-GP \citep{bilovs2023modeling} - one with an RNN for a score function and one with a transformer. For Diffusion-TS and CSPD-GP, we keep the model configurations as they were proposed in the authors' implementations for datasets with similar dimensions and number of data points. One exception to this is halving the batch size for transformer-based architectures due to memory constraints. We also halve the number of epochs to preserve the proposed number of training steps. As DDO has previously only been implemented on image-shaped data, we alter the code in the authors' GitHub implementation to generate samples of shape (\textit{time series length} x 1 x \textit{number of channels}). We always report the performance for the RBF kernel smoothness hyperparameter $\gamma \in [0.05, 0.2, 1, 5, 10]$ corresponding to the highest predictive score. We train the model for 300 epochs (see Appendix Section J of the DDO paper by \citet{lim2023score}) with a batch size of 32 and keep the remaining hyperparameters as proposed for the \textit{Volcano} dataset. Table \ref{table:model_size} shows the number of parameters and computation times for each model.
We use the publicly available code to run the benchmarks:
\begin{itemize}
    \item https://github.com/morganstanley/MSML \citep{morganstanley2024msml}
    \item https://github.com/Y-debug-sys/Diffusion-TS \citep{ydebugsys2024diffusionts}
    \item https://github.com/lim0606/ddo \citep{limgithub}
\end{itemize}

\paragraph{Experiments with the truncation level}
To produce \cref{fig:samples_lvls}, we use a dataset generated by \cref{eqn:lo-vol}, discretised to 25 points. To add non-trivial high-frequency components, we include sine and cosine waves of the first eight Fourier frequencies, with random amplitudes drawn from a standard normal distribution. To keep the model architecture consistent and avoid memory issues at higher truncation levels, we use the linearised attention mechanism from \citet{katharopoulos2020transformers} with a \texttt{relu} kernel, instead of the full attention mechanism. We also reduce the model's hidden size from 64 to 32.

\end{document}